\documentclass[hidelinks,onefignum,onetabnum]{siamart250211}

\usepackage{lipsum}
\usepackage{amsfonts}
\usepackage{graphicx}
\usepackage{epstopdf}
\usepackage{comment}
\usepackage{algorithmic}
\usepackage{float}
\usepackage{subcaption}
\usepackage{booktabs}

\ifpdf
  \DeclareGraphicsExtensions{.eps,.pdf,.png,.jpg}
\else
  \DeclareGraphicsExtensions{.eps}
\fi

\newsiamremark{remark}{Remark}
\newsiamremark{hypothesis}{Hypothesis}
\crefname{hypothesis}{Hypothesis}{Hypotheses}
\newsiamthm{claim}{Claim}
\newsiamremark{fact}{Fact}
\crefname{fact}{Fact}{Facts}

\headers{The Stochastic Occupation Kernel Method}{M. Wells, K. Lahouel, B. Jedynak}

\title{The Stochastic Occupation Kernel (SOCK) Method for Learning Stochastic Differential Equations}

\author{ Michael L. Wells 
    \thanks{Fariborz Maseeh Department of Mathematics and Statistics\\
    Portland State University
    Portland, OR 97201 
    (\email{mlwells@pdx.edu})} 
    \and
    Kamel Lahouel 
    \thanks{Translational Genomics Research Institute 
    Phoenix, AZ 85004 
    (\email{klahouel@tgen.org})}
    \and
    Bruno Jedynak 
    \thanks{Fariborz Maseeh Department of Mathematics and Statistics 
    Portland State University 
    (\email{bjedyna2@pdx.edu})} }

\usepackage{amsopn}

\ifpdf
\hypersetup{
  pdftitle={The Stochastic Occupation Kernel (SOCK) Method for Learning Stochastic Differential Equations},
  pdfauthor={Michael Wells, Kamel Lahouel, and Bruno Jedynak}
}
\fi

\begin{document}

\maketitle

\begin{abstract}
  We present a novel kernel-based method for learning multivariate stochastic differential equations (SDEs). The method follows a two-step procedure: we first estimate the drift term function, then the (matrix-valued) diffusion function given the drift. Occupation kernels are integral functionals on a reproducing kernel Hilbert space (RKHS) that aggregate information over a trajectory. Our approach leverages vector-valued occupation kernels for estimating the drift component of the stochastic process. For diffusion estimation, we extend this framework by introducing operator-valued occupation kernels, enabling the estimation of an auxiliary matrix-valued function as a positive semi-definite operator, from which we readily derive the diffusion estimate.
  This enables us to avoid common challenges in SDE learning, such as intractable likelihoods, by optimizing a reconstruction-error-based objective.  We propose a simple learning procedure that retains strong predictive accuracy while using Fenchel duality to promote efficiency.  We validate the method on simulated benchmarks and a real-world dataset of Amyloid imaging in healthy and Alzheimer's disease subjects.
\end{abstract}



\section{Introduction}
The modeling of time series data using dynamical systems is a common practice in the sciences and engineering.  Traditionally, a technician would select an appropriate parametric model using prior knowledge of the  system under investigation, and then they would fit the parameters with the available data.  In recent years, there has been an explosion in the development of methods for automatic system identification when there is limited prior knowledge of the system available, for example, see \cite{SINDy},  \cite{neuralODE_Chen2018}, \cite{LAHOUEL2024112971}.

When the dynamics of the system itself are thought to contain randomness, SDEs may be used \cite{sde_modeling}.  These are dynamical systems consisting of a deterministic {\em drift} and a stochastic  {\em diffusion}.  SDEs are frequently used in domains like physics \cite{SDEs_in_physics_Milstein2004}, biology \cite{SDEs_in_biology_Ditlevsen2012}, and finance \cite{SDEs_in_finance}, \cite{finance-Scholes2008}.  The SDE learning task is that of estimating the drift and diffusion.

There are a variety of existing methods for automatically inferring SDEs.  In \cite{DARCY2023133583} \cite{heinonen-sde}, the drift and diffusion are parametrized by Gaussian processes (GPs), while in \cite{DynGMA-Zhu}, they use neural networks.  The method of \cite{BISDE2022244} fits the model over a library of functions using sparse regression, similar to SINDy\cite{SINDy} for ordinary differential equations (ODEs).  Finally, the method of \cite{gEDMD-Klus} approximates the Koopman generator of the SDE from which one may obtain the drift and diffusion.  A different but related set of methods are neural SDEs \cite{neuralSDE-Kidger2021}, \cite{NeuralSDE-LatentSDE}, \cite{neural-sde_stable_sde_Oh2024}, which parametrize SDEs on latent space with neural networks.  These methods are the stochastic analogs of neural ODEs, first investigated in \cite{neuralODE_Chen2018}.  

Occupation kernels were pioneered in the work of \cite{Rosenfeld_2019} \cite{rosenfeld2024} and extended in \cite{rielly2025mock} as a tool for estimating the slope field of an ODE.  They are functions belonging to a reproducing kernel Hilbert space (RKHS) \cite{RKHS_book_Berlinet2011} which reduce the computation of an integral over a trajectory to an inner product in the RKHS.  Using a representer theorem, the process of optimizing a slope field becomes a non-parametric regression, similar to gradient matching.  However, gradient matching is a finite difference-based approach, and thus susceptible to error in the case of noisy data, while occupation kernels are finite element-based, making them more robust to noise.  A weak formulation of the occupation kernel method was investigated in \cite{rielly2025rockvariationalformulationoccupation}.

To adapt the concept of occupation kernels to the stochastic setting, we introduce an expected value over trajectories with fixed initial condition in the associated functionals.  We use the It\^{o} isometry to introduce a reconstruction-error-based loss function for the diffusion and present a novel, operator occupation kernel construction to optimize this loss.  We use the parametrization from \cite{nonnegative_funcs_marteau2020}, \cite{learning_psd_valued_Muzellec} and adapt the representer theorem for PSD matrix-valued functions therein to the setting of occupation kernels by accommodating integral operators in the place of evaluation operators.  In this way, we constrain our estimate of the squared diffusion $\sigma_0\sigma_0^T$ to be PSD-matrix-valued, lending it further validity.

The paper is organized as follows: We present the problem setting and notation in \cref{sec:notation}, the method in \cref{sec:method}, results of numerical experiments in \cref{sec:experiments}, 
and our conclusions in \cref{seC:conclusion}.

\section{Problem setup and notation}
\label{sec:notation}

Suppose we have a collection of $d$-dimen-sional snapshots of a trajectory at the times $t_0 < t_1 < \ldots < t_n$ of an It\^{o} SDE whose solution $x \colon [0, T] \times \Omega \to \mathbb{R}^d$ satisfies the following:
\begin{equation}
    dx_t = f_0(x_t) dt + \sigma_0(x_t) dW_t
    \label{eq:SDE_equation}
\end{equation}
where $f_0 \colon \mathbb{R}^d \to \mathbb{R}^d$ is the unknown {\em drift}, $\sigma_0 \colon \mathbb{R}^d \to \mathbb{R}^{d \times d}$ is the unknown {\em diffusion} and $W_t$ is $d$-dimensional Brownian motion.  We assume that strong solutions of the system exist and are unique for any choice of initial condition (see \cite[Theorem 5.2.1]{oksendal2014}).  Suppose also that we have $M$ realizations of these observations, which we denote by $\bigcup_{u=1}^M \{y_i^{(u)}\}_{i=0}^n$.  Thus, we have $y_i^{(u)} := x(t_i, \omega^{(u)})$ with $\omega^{(u)} \in \Omega$ for each $u=1,\ldots, M$.  We additionally require that $x_0 := x(0, \omega^{(u)})$ is fixed for all $u$.  Extending the following derivation to the case of multiple trajectories with multiple initial conditions is straight-forward.

Let $k \colon \mathbb{R}^d \times\mathbb{R}^d \to \mathbb{R}$ denote a positive-definite, scalar-valued kernel. Let $K \colon \mathbb{R}^d \times \mathbb{R}^d \to \mathbb{R}^{d \times d}$ denote a positive-definite, matrix-valued kernel.  If a kernel $K$ has the property that $K(x,y) = k(x,y)I_d$ for some scalar kernel $k$, we say that $K$ is {\em $I$-separable} \cite{rielly2025mock}.  Let $H$ be the RKHS corresponding to $K$.  
Let $Q_{[a,b]}[g(y_t^{(u)})]$ denote a quadrature of the integral $\int_a^b g(y_t^{(u)})dt$ using the $u^{th}$ realization of the trajectory in the training set.  Similarly let $Q_{[a,b]\times[c,d]} [g(y_s^{(u)}, y_t^{(v)})]$ be a double integral quadrature over the rectangle $[a, b] \times [c, d]$.

The proofs of the following propositions and theorems are deferred to the appendices for space considerations.

\section{Method}
\label{sec:method}
\subsection{Estimating the drift}
\label{sec:est_drift}
The integral form of \eqref{eq:SDE_equation} is given by
\begin{equation}
    x_{t_{i+1}} - x_{t_i} = \int_{t_i}^{t_{i+1}} f_0(x_t) dt + \int_{t_i}^{t_{i+1}}\sigma_0(x_t) dW_t
    \label{eq:SDE_integral_form}
\end{equation}
where $\int dW_t$ denotes the It\^{o} integral.  We take the conditional expectation of both sides given that the trajectory under consideration has fixed initial condition $x_0$:
\begin{equation}
    \mathbb{E}[x_{t_{i+1}} - x_{t_i} \vert x_0] = \mathbb{E}\left[\int_{t_i}^{t_{i+1}}f_0(x_t) dt \Bigg\vert x_0\right]
    \label{eq:SDE_cond_exp}
\end{equation}
which follows since the expectation of an It\^{o} integral is zero \cite{oksendal2014}.  This suggests that we minimize the following cost function over candidate drift estimates $f \in H$, an RKHS with kernel $K$:
\begin{equation}
    J_{drift}(f) = \frac{1}{n}\sum_{i=0}^{n-1}\left\|\mathbb{E}\left[\int_{t_i}^{t_{i+1}}f(x_t) dt \Bigg \vert x_0 \right] - \overline{y_{i+1} - y_i}\right\|^2_{\mathbb{R}^d} + \lambda_f \|f\|^2_H
    \label{eq:drift_cost}
\end{equation}
where $\overline{y_{i+1} -  y_i} := \frac{1}{M}\sum_{u=1}^M \left(y_{i+1}^{(u)} - y_i^{(u)}\right)$ is an empirical estimate of the left-hand side of \eqref{eq:SDE_cond_exp} and $\lambda_f > 0$ is a hyperparameter governing the amount of regularization to be applied. 
\begin{theorem}[Representer theorem]
\label{thm:drift_representer_thm}
Suppose that the following regularity condition holds:
\begin{equation}
    \mathbb{E}\left[\int_{t_i}^{t_{i+1}}\mbox{\rm Tr}(K(x_t, x_t)) dt \Bigg \vert x_0 \right] < \infty
\end{equation}
for $i=0,\ldots, n-1$ where $\mbox{\rm Tr}(\cdot)$ denotes the matrix trace.  Then there exists a minimizer $f^* \in H$ of \eqref{eq:drift_cost} given as a linear combination of functions $L_i^*$:
\begin{equation}
    f^* = \sum_{i=0}^{n-1}L_i^* \alpha_i
\end{equation}
where $\alpha_i \in \mathbb{R}^d$  and $L_i^* \colon \mathbb{R}^d \to \mathbb{R}^{d \times d}$ is the function defined by
\begin{equation}
    \label{eq:drift_occ_kern_definition}
    L_i^*(x) := \mathbb{E}\left[\int_{t_i}^{t_{i+1}}K(x, x_t) dt \Bigg\vert x_0\right]
\end{equation}
for $i=0,\ldots, n-1$.  The coefficients $\alpha_i$ may be found by solving
\begin{equation}
\label{eq:linear_system_for_drift}
    (L^* + n\lambda_f I_{nd})\alpha = \overline{\Delta y}
\end{equation}
where $L^*$ is an $(nd, nd)$ block matrix of $(d,d)$ blocks whose $(k, l)$ block is given by
\begin{equation}
    [L^*]_{k, l} := \mathbb{E}_{\omega_1 \in \Omega, \omega_2 \in \Omega}\left[\int_{t_k}^{t_{k+1}}\int_{t_l}^{t_{l+1}}K(x_s(\omega_1), x_t(\omega_2)) dt ds \Bigg \vert x_0 \right]
\end{equation}
and $\overline{\Delta y} \in \mathbb{R}^{nd}$ is the concatenation of $\overline{y_{i+1} - y_i}$ for $i=0,\ldots, n-1$.
\label{theorem:representer_thm}
\end{theorem}
The functions $L_i^*$ are Riesz representers for the vector-valued functionals $L_i \colon f \mapsto \mathbb{E}\left[\int_{t_i}^{t_{i+1}} f(x_t) dt \Bigg \vert x_0\right]$.  The regularity condition assures that the $L_i$ are bounded which guarantees the existence of the dual functions $L_i^*$.  Indeed, it can be shown that for $v \in \mathbb{R}^d$, we have $L_i(f)^Tv = \langle f, L_i^* v \rangle_H$.  We refer to functions of the form $x \mapsto L_i^*(x) v$ for $v \in \mathbb{R}^d$ as {\em occupation kernels} \cite{rielly2025mockalgorithmlearningnonparametric}.  See appendix \ref{sec:drift_derivation} for more details.  A summary is presented in algorithm \ref{alg:SOCK drift_learning}.

\subsection{Estimating the diffusion}
\label{sec:est_diffusion}

By taking outer products and conditional expectations in equation \eqref{eq:SDE_integral_form}, we may write
\begin{multline}
    \mathbb{E}\left[\left(x_{t_{i+1}} - x_{t_i} - \int_{t_i}^{t_{i+1}} f_0(x_t) dt\right)\left(x_{t_{i+1}} - x_{t_i} - \int_{t_i}^{t_{i+1}} f_0(x_t) dt\right)^T\Bigg \vert x_0\right] =\\ \mathbb{E}\left[\left(\int_{t_i}^{t_{i+1}}\sigma_0(x_t) dW_t\right)\left(\int_{t_i}^{t_{i+1}}\sigma_0(x_t) dW_t\right)^T \Bigg \vert x_0 \right]=
    \mathbb{E}\left[\int_{t_i}^{t_{i+1}}\sigma_0\sigma_0^T(x_t) dt \Bigg \vert x_0\right]
    \label{eq:SDE_outer_product}
\end{multline}
using the It\^{o} isometry \cite{oksendal2014} for the last equality.
Let us define
\begin{equation}
    z_i := \frac{1}{M}\sum_{u=1}^M \left(y_{i+1}^{(u)} - y_i^{(u)} - Q_{[t_i, t_{i+1}]}[f(y_t^{(u)})]\right)\left(y_{i+1}^{(u)} - y_i^{(u)} - Q_{[t_i, t_{i+1}]}[f(y_t^{(u)})] \right)^T
\end{equation}
for $i=0,\ldots, n-1$ using the previously estimated drift $f$.  This is an empirical estimate of the left-hand side of  \eqref{eq:SDE_outer_product}. 
We choose to estimate the function $\sigma_0\sigma_0^T$ rather than $\sigma_0$ directly.  Thus, we minimize the following over functions $a \in {{H'}^{d \times d}}$, a suitable space of $(d,d)$ matrix-valued functions:
\begin{equation}
    J_{diff}(a) = \frac{1}{n}\sum_{i=0}^{n-1}\left\|\mathbb{E}\left[\int_{t_i}^{t_{i+1}}a(x_t) dt \Bigg\vert x_0\right] - z_i\right\|^2_{\mathbb{R}^{d \times d}} + \lambda_\sigma \|a\|^2_{{{H'}^{ d \times d}}}
    \label{eq:diff_outer_cost}
\end{equation}
subject to the constraint that $a(x) \succeq 0$ for each $x \in \mathbb{R}^d$.  Performing constrained optimization over $a$ rather than unconstrained optimization over $\sigma$ allows us to avoid the quartic nature of the cost function in $\sigma$.

We discuss two methods for estimating $a$: an implicit and an explicit kernel method.  For the sparse data regime, the implicit kernel method is preferred.  However, its complexity has a non-linear dependence on $n$, the number of observations in the training set.  The explicit kernel method is useful for larger datasets, having no restrictive dependence on the dataset size; only a non-linear dependence on the number of features in the feature map.

\subsubsection{Implicit kernel}
We let $H$ denote an RKHS with scalar kernel $k$.  
Following \cite{learning_psd_valued_Muzellec}, we suppose that our estimate $a$ of $\sigma_0\sigma_0^T$ is given by
\begin{equation}
    a(x) := \left[\langle k(\cdot, x), C_{kl} k(\cdot, x) \rangle_H\right]_{k,l=1}^d
\end{equation}
where each $C_{kl} \in \mbox{Hom}_{HS}(H)$, the space of Hilbert-Schmidt operators on $H$.    Notate $\varphi(x) := k(\cdot, x)$ and $H^d := H \times \ldots \times H$, the $d$-fold product of $H$ with itself.  We define $C \colon H^d \to H^d$ as the $(d,d)$ matrix of operators $C_{kl}$.
We define $\hat{\varphi}(x) := I_d \otimes \varphi(x) \in \mathcal{L}(\mathbb{R}^d, H^d)$, the space of linear maps from $\mathbb{R}^d$ to $H^d$.  Note that  $a(x) = \hat{\varphi}(x)^* C \hat{\varphi}(x)$, where $\hat{\varphi}^*$ denotes the adjoint of $\hat{\varphi}$. We require that $C \succeq 0$ in the sense that
\begin{equation}
        \mathbf{g}^*
        C\mathbf{g}  = \sum_{k,l=1}^d \langle g_k, C_{kl} g_l \rangle_H \geq 0, \;\; \mbox{ for all } \mathbf{g} \in H^d
\end{equation}
which ensures that $a(x) \succeq 0$ for all $x \in \mathbb{R}^d$.  For $c, d \in H$, we denote by $c \otimes d$ the rank-one operator $H \to H$ given by $(c\otimes d)e := \langle d, e \rangle_H c$.

\begin{proposition}
\label{prop:operator_occ_kerns}
   Suppose the following regularity conditions hold:
   \begin{align}
    &\mbox{\rm i)  }\mathbb{E}\left[\int_{t_i}^{t_{i+1}}k(x_t, x_t)^2 dt \Bigg \vert x_0 \right] < \infty,\;\;\; i=0,\ldots, n-1,\\
       &\mbox{\rm ii)  }\{D \in \mbox{\rm Hom}_{HS}(H) \colon D^* = D\} = \overline{\mbox{\rm span}\{\varphi(x)\otimes \varphi(x) \colon x \in \mathbb{R}^d\}}
   \end{align}
   where the closure is taken in $\mbox{\rm Hom}_{HS}(H)$.  Then we have the following identity:
    \begin{equation}
        \mathbb{E}\left[\int_{t_i}^{t_{i+1}}a(x_t) dt \Bigg \vert x_0 \right] = \left[\left \langle \mathbb{E}\left[\int_{t_i}^{t_{i+1}}\varphi(x_t) \otimes \varphi(x_t) dt \Bigg \vert x_0\right], C_{kl} \right\rangle_{HS}\right]_{k,l=1}^d 
    \end{equation}
    where $\langle \cdot, \cdot \rangle_{HS}$ denotes the Hilbert-Schmidt inner product.  Thus, the cost in \eqref{eq:diff_outer_cost} may be written as a cost function over operators $C \in \mbox{\rm Hom}_{HS}(H^d)$:
    \begin{equation}
        J_{diff}(C) = \frac{1}{n}\sum_{i=0}^{n-1}\left\|\left[\langle M_i, C_{kl} \rangle_{HS}\right]_{k,l=1}^d - z_i\right\|^2_{\mathbb{R}^{d \times d}} + \lambda_\sigma \|C\|^2_{HS}
    \end{equation}
    subject to the constraint that $C \succeq 0$ where  $M_i := \mathbb{E}\left[\int_{t_i}^{t_{i+1}} \varphi(x_t) \otimes \varphi(x_t) dt \Bigg \vert x_0 \right] \in \mbox{\rm Hom}_{HS}(H)$ for $i=0,\ldots,n-1$.
  
\end{proposition}
The regularity conditions ensure that the space of functions $x \mapsto \langle k(\cdot, x), Dk(\cdot, x)\rangle_H$ for $D \in \mbox{Hom}_{HS}(H), D^* = D$ is an RKHS and that the matrix-valued functionals $a \mapsto \mathbb{E}\left[\int_{t_i}^{t_{i+1}} a(x_t) dt \Bigg \vert x_0\right]$ are bounded.  As in the case of the drift, this guarantees the existence of Riesz representers $M_i$ which reduce the computation of the expected integral over a trajectory to an inner product in a Hilbert space.  We refer to the operators $M_i$ as {\em occupation kernels}.

To apply the representer theorem of \cite{learning_psd_valued_Muzellec}, we replace the $M_i$ with finite-rank approximations, computed as  
\begin{equation}
    \hat{M}_i := \frac{t_{i+1} - t_i}{2M}\sum_{u=1}^M\left[\varphi(y_i^{(u)}) \otimes \varphi(y_i^{(u)}) + \varphi(y_{i+1}^{(u)})\otimes \varphi(y_{i+1}^{(u)})\right]
\end{equation}
where we have used a trapezoid-rule integral quadrature and Monte Carlo estimate of the expectation.  Thus, we minimize
    \begin{equation}
    \label{eq:diff_cost_with_M_i_hat}
        J'_{diff}(C) = \frac{1}{n}\sum_{i=0}^{n-1}\left\|\left[\langle \hat{M}_i, C_{kl} \rangle_{HS}\right]_{k,l=1}^d - z_i\right\|^2_{\mathbb{R}^{d \times d}} + \lambda_\sigma \|C\|^2_{HS}
    \end{equation}
    subject to $C \succeq 0$.
    \begin{theorem}[Representer theorem]
    \label{thm:diff_representer_thm}
        Suppose that the conditions of proposition \ref{prop:operator_occ_kerns} hold. For each $C \in \mbox{\rm Hom}_{HS}(H^d)$, we have that
        \begin{equation}
            J'_{diff}(\Pi_n C \Pi_n) \leq J'_{diff}(C)
        \end{equation}
        where $\Pi_n \colon H^d \to H^d$ is component-wise orthogonal projection onto the finite-dimen-sional subspace spanned by $k(\cdot, y_i^{(u)})$ for $i=0,\ldots,n$ and $u=1,\ldots, M$.  Thus, we may find a minimizer of \eqref{eq:diff_cost_with_M_i_hat} of the form $\Pi_n C \Pi_n$ for some $C \in \mbox{\rm Hom}_{HS}(H^d)$.
    \end{theorem}
We adapt the proof in \cite{learning_psd_valued_Muzellec} to our setting for this representer theorem.  It is essential for the $\hat{M}_i$ to be finite rank in order for this result to hold.
    
Let us define the map $\psi \colon H \to \mathbb{R}^{(n+1)M}$ by \begin{equation}
    \psi g := \begin{bmatrix}
       g(y_0^{(1)}), & \ldots&, g(y_n^{(M)}) 
    \end{bmatrix}^T
\end{equation}
and $G := [k(y_i^{(u)}, y_j^{(v)})] \in \mathbb{R}^{(n+1)M \times (n+1)M}$.
Let $R \in \mathbb{R}^{r \times (n+1)M}$ be such that $R^T R = G$ where $r := \mbox{\rm rank}(G)$.  Then let
\begin{equation}
    \label{eq:gamma_definition}
    \gamma := (RR^T)^{-1}R\psi \in \mathcal{L}(H, \mathbb{R}^r),\;\;\;
    \hat{\gamma} := I_d \otimes \gamma \in \mathcal{L}(H^d, \mathbb{R}^{rd})
\end{equation}

\begin{theorem}
\label{thm:isometry}
    Suppose that the conditions of proposition \ref{prop:operator_occ_kerns} hold.  There is an isometry between the space of operators $ \{\Pi_n C \Pi_n \colon C \in \mbox{\rm Hom}_{HS}(H^d), C^* = C\}$ and the space $\{\hat{\gamma}^* A \hat{\gamma} \colon A \in \mathbb{R}^{rd \times rd}, A^T = A\}$.  Thus, the minimization problem in \eqref{eq:diff_cost_with_M_i_hat} is equivalent to a minimization problem over symmetric matrices  $A \in \mathbb{R}^{rd \times rd}$:
    \begin{equation}
        J'_{diff}(A) = \frac{1}{n}\sum_{i=0}^{n-1}\left\|[\langle N_i, A_{kl} \rangle_{\mathbb{R}^{r \times r}}]_{k,l=1}^d - z_i\right\|^2_{\mathbb{R}^{d \times d}} + \lambda_\sigma \|A\|^2_{\mathbb{R}^{rd \times rd}}
        \label{eq:diff_outer_cost_with_A}
    \end{equation}
    subject to $A \succeq 0$ where $N_i := \frac{t_{i+1} - t_i}{2M}\sum_{u=1}^M\left[\gamma(y_i^{(u)}) \gamma(y_i^{(u)})^T + \gamma(y_{i+1}^{(u)})\gamma(y_{i+1}^{(u)})^T\right] \in \mathbb{R}^{r \times r}$ and $\gamma(x) := \gamma \varphi(x)$.
    
\end{theorem}
The isometry may be proved using the relations $\hat{\gamma}\hat{\gamma}^* = I_{rd}$ and $\hat{\gamma}^* \hat{\gamma} = \Pi_n$ and noting that $\hat{\gamma} C\hat{\gamma}^* \in \mathbb{R}^{rd \times rd}$. The remainder of the proof is merely substituting $\hat{\gamma}^* A\hat{\gamma}$ in the relevant places and simplifying.  See appendix \ref{sec:implicit_diffusion_appendix} for more details.

Following \cite{learning_psd_valued_Muzellec}, we minimize the Fenchel dual of the cost instead.
\begin{theorem}
\label{thm:fenchel_dual_implicit_kernel}
   Suppose that the conditions of proposition \ref{prop:operator_occ_kerns} hold.  The minimizer of the Fenchel dual of \eqref{eq:diff_outer_cost_with_A} is given by
    \begin{multline}
        \label{eq:Fenchel_dual_cost}
        \beta^* := \\
         \underset{\beta \in \mathbb{R}^{n \times d \times d}}{\rm argmin}
        \left\{\frac{n}{4}\|\beta\|^2_{\mathbb{R}^{n \times d \times d}} + \langle \beta, z \rangle_{\mathbb{R}^{n \times d \times d}} + \frac{1}{4\lambda_\sigma}\left\|\left[\left[\sum_{i=0}^{n-1}N_i \beta_{i, kl}\right]_{k,l=1}^d\right]_-\right\|^2_{\mathbb{R}^{rd \times rd}}\right\}
    \end{multline}
    where $[A]_- = [U\Sigma U^T]_- := U\max\{-\Sigma, 0\}U^T$ and $z \in \mathbb{R}^{n \times d \times d}$ is the concatenation of the $z_i$.  Then the optimal $A^*$ minimizing \eqref{eq:diff_outer_cost_with_A} is given by 
    \begin{equation}
        A^* = \frac{1}{2\lambda_\sigma}\left[\left[\sum_{i=0}^{n-1} N_i \beta_{i, kl}^*\right]_{k,l=1}^d \right]_-
    \end{equation}

\end{theorem}
This result is a direct application of Fenchel's duality theorem \cite{rockafellar2015convex}.  One may express \eqref{eq:diff_outer_cost_with_A} as an unconstrained problem by defining a function $\Omega(A) := \lambda_\sigma \|A\|^2_{\mathbb{R}^{rd \times rd}}$ if $A \succeq 0$ and $\infty$ otherwise and writing $J'_{diff}(A) = \theta(RA) + \Omega(A)$ using the linear operator $R \colon A \mapsto \left[\left[\langle N_i, A_{kl} \rangle_{\mathbb{R}^{r \times r}}\right]_{k,l=1}^d\right]_{i=0}^{n-1} \in \mathbb{R}^{n \times d \times d}$ and function $\theta(\beta) := \frac{1}{n}\|\beta - z\|^2_{\mathbb{R}^{n \times d \times d}}$.  Fenchel's duality theorem asserts that the dual cost is then given by $\theta^*(\beta) + \Omega^*(-R^*\beta)$ where $\theta^*, \Omega^*$ are the Fenchel conjugates of $\theta, \Omega$, respectively, and $R^*$ is the adjoint of $R$.  If $\beta^*$ is optimal for the dual problem, then the optimal $A^*$ for the primal problem will be $\nabla \Omega^*(-R^*\beta^*)$, which can be shown using the relation $y \in \partial g(x) \Leftrightarrow x \in \partial g^*(y)$ and other elementary properties of subgradients.

The right side of \eqref{eq:Fenchel_dual_cost} may be minimized efficiently using the FISTA algorithm \cite{FISTA_beck2009}.  The optimal $a$ is then evaluated as  $a(x) = \hat{\gamma}(x)^T A^* \hat{\gamma}(x)$.  We derive an estimate of $\sigma_0$ by $\sigma(x) := U(x)\sqrt{\Sigma(x)} U(x)^T$ where $a(x) = U(x)\Sigma(x)U(x)^T$ is an eigendecomposition of the output.  Note that $\sigma_0$ is only unique up to a rotation $O$ since $\sigma_0(x) \varepsilon$ has the same distribution as $\sigma_0(x) O\varepsilon$ for $\varepsilon \sim N(0, I)$.  Thus, solutions for the SDE with diffusion $\sigma_0$ and the SDE with diffusion $\sigma_0 O$ will have the same law.  However, the two functions $f_0$ and  $\sigma_0\sigma_0^T$ are unique and fully characterize the SDE.  A summary is presented in algorithm \ref{alg:SOCK diffusion_learning}.

\begin{algorithm} 
    \caption{SOCK drift learning:}
    \label{alg:SOCK drift_learning}
\begin{algorithmic}[1]
    \REQUIRE \textbf{Training data}: $\{y_i^{(u)}\} \subset \mathbb{R}^d, i=0 \ldots n, u=1,\ldots, M$, $\lambda_f > 0$. 
    \FOR{$k=0 \ldots n-1, l=0 \ldots n-1$} 
    \STATE $L^*_{kl} \gets \frac{1}{M^2} \sum_{u=1}^M \sum_{v=1}^M Q_{[t_k, t_{k+1}]\times [t_l, t_{l+1}]}[K(y_s^{(u)}, y_t^{(v)})] \in \mathbb{R}^{d \times d}$ 
    \ENDFOR
    \STATE Set $L^*$ the $(nd, nd)$ block matrix with $(k,l)$ block $L_{kl}^*$.
    \STATE Solve the linear system $(L^*+\lambda_f I_{nd})\alpha = \overline{\Delta y}$ for $\alpha$
    \STATE \textbf{Return}: $f \colon x \mapsto \sum_{i=0}^{n-1}L_i^*(x) \alpha_i$ with $L_i^*$ as in \eqref{eq:drift_occ_kern_definition}.
\end{algorithmic}
\end{algorithm}

\begin{algorithm} 
    \caption{SOCK diffusion learning (implicit):}
    \label{alg:SOCK diffusion_learning}
\begin{algorithmic}[1]
    \REQUIRE \textbf{Training data}: $\{y_i^{(u)}\} \subset \mathbb{R}^d, i=0 \ldots n, u=1,\ldots, M$,  $f \in H$,  $\lambda_\sigma > 0$. 
    \STATE $\beta^* \gets \underset{\beta \in \mathbb{R}^{n \times d \times d}}{\rm argmin} J_{dual}(\beta)$ where $J_{dual}$ is the function on the right side of \eqref{eq:Fenchel_dual_cost}.
    
    \STATE $A^* \gets \frac{1}{2\lambda_{\sigma}} \left[\left[\sum_{j=0}^{n-1}N_j \beta^*_{j, kl}\right]_{k,l=1}^d\right]_-$
    \STATE \textbf{Return}: $a \colon x \mapsto \hat{\gamma}(x)^T A^* \hat{\gamma}(x)$ with $\hat{\gamma}$ as in \eqref{eq:gamma_definition}.
\end{algorithmic}
\end{algorithm}

\subsubsection{Explicit kernel}
\label{sec:diff_explicit_kernel_main_paper}
As an alternative, suppose we parametrize the function $a$ according to
\begin{equation}
    a(x) := \hat{\varphi}(x)^T Q \hat{\varphi}(x)
\end{equation}
where $\hat{\varphi} := I_d \otimes \varphi$ for a feature vector $\varphi \colon \mathbb{R}^d \to \mathbb{R}^p$.  Thus, $\hat{\varphi}(x) \in \mathbb{R}^{pd \times d}$ for each $x \in \mathbb{R}^d$.  Here, $Q$ is a $(pd, pd)$ matrix, which we require to be positive semi-definite.  Then \eqref{eq:diff_outer_cost} becomes an optimization over matrices $Q$ given by
\begin{equation}
    J_{diff}(Q) = \frac{1}{n}\sum_{i=0}^{n-1}\left\|\mathbb{E}\left[\int_{t_i}^{t_{i+1}}\hat{\varphi}(x_t)^T Q \hat{\varphi}(x_t) dt\Bigg \vert x_0 \right] - z_i\right\|^2_{\mathbb{R}^{d \times d}} + \lambda_{\sigma}\|Q\|^2_{\mathbb{R}^{pd \times pd}}
\end{equation}
subject to the constraint $Q \succeq 0$.

In this case, the occupation kernels will be given by
\begin{equation}
    M_i := \mathbb{E}\left[\int_{t_i}^{t_{i+1}}\varphi(x_t)\varphi(x_t)^T dt \Bigg \vert x_0\right] \in \mathbb{R}^{p \times p}
\end{equation}
and the cost reduces to
\begin{equation}
    \label{eq:diff_explicit_cost}
    J_{diff}(Q) = \frac{1}{n}\sum_{i=0}^{n-1}\left\|\left[\langle M_i, Q_{kl} \rangle_{\mathbb{R}^{p \times p}}\right]_{k,l=1}^d - z_i\right\|^2_{\mathbb{R}^{d\times d}} + \lambda_{\sigma}\|Q\|^2_{\mathbb{R}^{pd \times pd}}
\end{equation}
subject to $Q \succeq 0$.  We may then apply theorem \ref{thm:fenchel_dual_implicit_kernel} as in the implicit kernel case to derive an optimization problem that is efficient to solve.  The Fenchel dual to \eqref{eq:diff_explicit_cost} is a function of $\beta \in \mathbb{R}^{n \times d \times d}$ given by
\begin{multline}
    J'_{dual}(\beta) =\\
    \frac{n}{4}\|\beta\|^2_{\mathbb{R}^{n \times d \times d}} + \langle \beta, z \rangle_{\mathbb{R}^{n \times d \times d}} + \frac{1}{4\lambda_{\sigma}}\left\|\left[\left[\sum_{i=0}^{n-1}M_i \beta_{i, kl}\right]_{k,l=1}^d\right]_-\right\|^2_{\mathbb{R}^{pd \times pd}}
\end{multline}See appendix \ref{sec:derivation_outer_diff} for additional details.

Alternatively, we may minimize \eqref{eq:diff_explicit_cost} directly using projected gradient descent.  Thus, we iteratively update $Q$ according to 
\begin{align}
W^{(n+1)} &\gets Q^{(n)} - \alpha\nabla J_{diff}(Q^{(n)})\\ Q^{(n+1)} &\gets \Pi W^{(n+1)}
\end{align}
where $\Pi$ is orthogonal projection onto the space of positive semi-definite $(pd, pd)$ matrices and $\alpha > 0$ is the learning rate.  Since \eqref{eq:diff_explicit_cost} with constraint $Q \succeq 0$ is a convex problem and $\nabla J_{diff}$ is Lipschitz continuous, the iterates will converge to the global solution provided that $\alpha$ is chosen correctly.  The optimal learning rate is given by $1/\ell$ where $\ell$ is the Lipschitz constant of $\nabla J_{diff}$.  This may be computed analytically.  Note that projected gradient descent may also be used for minimizing \eqref{eq:diff_outer_cost_with_A} over $(rd, rd)$ matrices $A$ with $A \succeq 0$ although we do not discuss the details.

\subsection{Computational complexity}

\subsubsection{Drift}
We restrict our attention to the case of $I$-separable kernels given by 
\begin{equation}
    K(x,y) = k(x,y)I_d
\end{equation}
The linear system in \eqref{eq:linear_system_for_drift} is equivalent to an $(n,n)$ system with an $(n, d)$ right-hand side in this case.  Thus, we have $\mathcal{O}(n^2M^2)$ to compute the matrix $L^*$, and $\mathcal{O}(dn^3)$ to solve the system.  The method is linear in $d$. 
\subsubsection{Diffusion (implicit)}
In the following, assume for simplicity that $M=1$.  That is, we consider the case of one realization of a trajectory per initial condition in the training set.  Assume also that the functions $k(\cdot, y_i^{(u)})$ are linearly independent (and thus $r=n$).  We use FISTA \cite{FISTA_beck2009} to minimize \eqref{eq:Fenchel_dual_cost} with respect to the tensor $\beta \in \mathbb{R}^{n \times d \times d}$.  Let 
\begin{equation}
h(\beta) := \frac{1}{4\lambda_\sigma}\left\|\left[\left[\sum_{i=0}^{n-1} N_i \beta_{i,kl}\right]_{k,l=1}^d \right]_-\right\|^2
\end{equation}
The complexity of one step of FISTA in this case reduces to the complexity of computing $\nabla h(\beta)$.  Let 
\begin{equation}
Y := \left[\left[\sum_{i=0}^{n-1} N_i \beta_{i,kl}\right]_{k,l=1}^d\right]_- \in \mathbb{R}^{nd \times nd}    
\end{equation}
Then we may show that  $\partial h/\partial \beta_{i,kl} = \frac{1}{2\lambda_\sigma} \langle -N_i, Y_{k, l} \rangle_{\mathbb{R}^{n \times n}}$,  where $Y_{k, l}$ is the $(k,l)$ block of $Y$ of size $(n,n)$. Thus, the complexity of computing $\nabla h(\beta)$ given $Y$ is $\mathcal{O}(d^2n^3)$.  If the algorithm takes $T$ iterations to converge, then the complexity is $\mathcal{O}(Td^2n^3)$.  Computing all of the $N_i$ costs $\mathcal{O}(n^3)$.  Computing the eigendecomposition to obtain $Y$ has a cost of $\mathcal{O}(d^3n^3)$.  Thus, the total complexity is given by $\mathcal{O}(Td^3n^3)$.
\subsubsection{Diffusion (explicit)}
As before, assume for simplicity that $M=1$.  We use FISTA to minimize the cost function with respect to the tensor $\beta \in \mathbb{R}^{n \times d \times d}$.  Let 
\begin{equation}
h(\beta) := \frac{1}{4\lambda_\sigma}\left\|\left[\left[\sum_{i=0}^{n-1} M_i \beta_{i,kl}\right]_{k,l=1}^d \right]_-\right\|^2
\end{equation}
Like the implicit kernel case, the complexity of one step of FISTA reduces to the complexity of computing $\nabla h(\beta)$.  Let 
\begin{equation}
Z := \left[\left[\sum_{i=0}^{n-1} M_i \beta_{i,kl}\right]_{k,l=1}^d\right]_- \in \mathbb{R}^{pd \times pd}    
\end{equation}
Then we may show that  $\partial h/\partial \beta_{i,kl} = \frac{1}{2\lambda_\sigma} \langle -M_i, Z_{k, l} \rangle_{\mathbb{R}^{p \times p}}$,  where $Z_{k, l}$ is the $(k,l)$ block of $Z$ of size $(p,p)$. Thus, the complexity of computing $\nabla h(\beta)$ given $Z$ is $\mathcal{O}(d^2p^2n)$.  If the algorithm takes $T$ iterations to converge, then the complexity is $\mathcal{O}(Td^2p^2n)$.  Computing all of the $M_i$ costs $\mathcal{O}(p^2n)$.  Computing the eigendecomposition to obtain $Z$ has a cost of $\mathcal{O}(d^3p^3)$.  Therefore, the total complexity is given by $\mathcal{O}(Td^3p^3 + Td^2p^2n)$.

We also consider the complexity of projected gradient descent.  We see that the cost of computing $\nabla J_{diff}(Q)$ is $\mathcal{O}(d^2p^2n)$ since there are $d^2n$ inner products of $(p,p)$ matrices to compute.  The cost of projecting a $(pd, pd)$ matrix to the cone of positive semi-definite matrices is equivalent to that of taking an eigendecomposition and is thus given by  $\mathcal{O}(d^3p^3)$.  Therefore, if the algorithm takes $T$ iterations to converge, the total complexity is given by $\mathcal{O}(Td^3p^3 + Td^2p^2n)$, coinciding with FISTA.
\subsection{Techniques to reduce the complexity for large datasets}

While we do not describe how to do so in this paper, one may use an explicit kernel for the drift, which reduces the complexity to a linear dependence on $n$.  The standard Nystr\"{o}m method \cite{drineas2005nystrom} may be applied to the matrix $L^*$ in the drift estimation, as well.  To apply Nystr\"{o}m in the case of the implicit kernel diffusion method, we would first select an appropriate subset of the observations $y_i^{(u)}$ of size $m << n$, labeling them $\tilde{y}_j$.  When computing the approximations $\hat{M}_i$ of the operator occupation kernels (as described in section \ref{sec:est_diffusion}), we orthogonally project the $\varphi(y_i^{(u)})$ appearing in the original definition to the space $V := \mbox{\rm span}\{\varphi(\tilde{y}_j) \colon j=1,\ldots, m\}$.  This enables us to use a feature vector $\tilde{\psi} \colon H \to \mathbb{R}^m$ such that $\tilde{\psi}(x) \in \mathbb{R}^m$ rather than $\mathbb{R}^{(n+1)M}$.  In the representer theorem (\cref{thm:diff_representer_thm}), we may then substitute the projection operator $\Pi_m$, component-wise projection onto $V$, for $\Pi_n$.  This has the effect of making the complexity linear in $n$.  Alternatively, one may parametrize the function $a(x)$ with an explicit kernel as in section \ref{sec:diff_explicit_kernel_main_paper}.
\section{Experiments}
\label{sec:experiments}
We performed numerical experiments on simulated datasets, a real dataset consisting of brain amyloid levels of healthy and Alzheimer's disease subjects, and a dataset generated by a stochastic SIR model.  All experiments were conducted on a laptop with the Intel Core Ultra 9 Processor 185H, 64 GB RAM, and an NVIDIA GPU with 8GB of memory.  All expectations and integrals appearing in sections \ref{sec:est_drift} and \ref{sec:est_diffusion} were estimated with Monte Carlo and trapezoid-rule quadratures, respectively.

We use the Gaussian kernel for the drift and Gaussian Fourier features for the diffusion as a default choice of kernels.  Note that performance may be improved by selecting a more appropriate kernel if one happens to possess prior knowledge of the system under investigation.  In several of the simulated experiments, we chose a kernel using this prior knowledge.  These results appear in italics.

\subsection{Datasets}
\label{sec:datasets}
For the simulated SDE datasets, we integrated trajectories using the Euler-Maruyama method with a small step-size ($\tau = 0.0001$) starting from initial conditions randomly selected from a box centered at the origin.  We generated points at times $t_i := ih$ for $h > 0$ chosen suitably small and $i = 0, \ldots, n$.  Only one realization of each trajectory was generated (thus, $M = 1$ in the notation of section \ref{sec:notation}). 
\begin{enumerate}
    \item {\bf Geometric Brownian motion}:
    
    Dynamics given by 
    \begin{equation}
    dx_t = x_t dt + 0.3 x_t dW_t
    \end{equation}
    in one dimension.  Each trajectory had 101 equispaced points starting from time $0$ until time $1$.  There were 70 training trajectories, 10 validation trajectories (for hyperparameter tuning), and 20 test trajectories.  Thus, there were a total of $8080$ observations in the combined training and validation sets.
   
    \item {\bf Exponential dynamics}:  Non-linear dynamics given by 
    \begin{equation}
    dx_t = \exp(-x_t^2) dt + 0.3 \exp(-x_t^2)dW_t
    \end{equation}
    in one dimension.  There were 11 equispaced points per trajectory from time 0 until time 1.  There were 42 training trajectories, 6 validation trajectories, and 12 test trajectories.  There were a total of 528 observations in the training and validation sets together.
    \item {\bf Dense matrix-valued diffusion}: Two-dimensional dynamics given by 
    \begin{equation}
    dx_t = x_t dt + 0.3 Ax_t b^T dW_t
    \end{equation}
    where $A \in \mathbb{R}^{2 \times 2}$ and $b \in \mathbb{R}^2$ had entries drawn from a random uniform distribution over $(-1, 1)$.  Thus, the stochastic term in the dynamics had correlated entries.  There were 70 training trajectories, 10 validation trajectories, and 20 test trajectories.  Each trajectory had 11 equispaced points starting from time 0 until time 1.  Thus, there were 880 observations in the combined training and validation sets.
   
    \item {\bf Stochastic Lorenz 96-10}:  Ten-dimensional SDE given by
    \begin{equation}
        dx_t^i = \left((x_t^{i+1} - x_t^{i-2})x_t^{i-1} - x_t^i + F\right) dt + 0.3 x_t^i dW_t^i
    \end{equation}
    for $i=1,\ldots, 10$ with $x_t^{-1} := x_t^{9}, x_t^{0} := x_t^{10}, x_t^{11} := x_t^1$ and $F = 8$.  There were 101 equispaced points per trajectory starting at time 0 until time 1.  There were 70 training trajectories, 10 validation trajectories, and 20 test trajectories.  There were a total of 8080 observations in the combined training and validation sets.  This is a stochastic version of the chaotic Lorenz96 system in ten dimensions \cite{lorenz1996predictability}.
 
\begin{figure}
    \centering
    \includegraphics[width=1.0\linewidth]{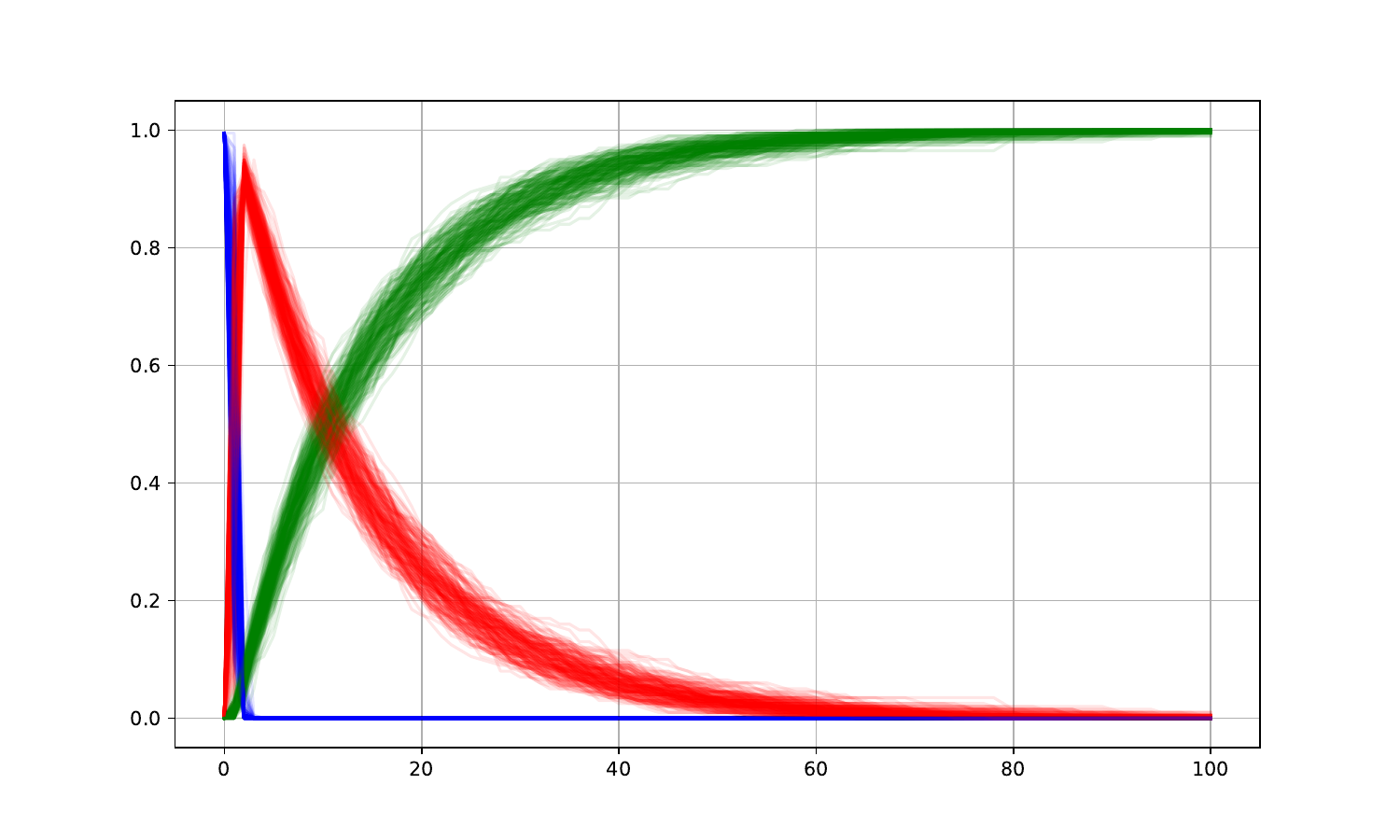}
    \caption{Stochastic SIR dataset.  Blue is susceptible (S), red is infected (I), and green is recovered (R).}
    \label{fig:stochastic_SIR}
\end{figure}
\end{enumerate}
We additionally obtained results on a real, medical dataset and a dataset generated by an agent-based model:
\begin{enumerate}
    \item {\bf Brain amyloid levels}:
    Amyloid levels were measured in healthy and Alz-heimer's disease subjects at multiple time points, yielding time series data in one dimension \cite{WRAP_johnson2018}.  There were 208 trajectories in the training set, 29 in the validation set, and 61 in the test set.  The trajectories had an average of 2.96 points on them.  
    \item {\bf Stochastic SIR model}:  We used the Gillespie stochastic simulation algorithm \cite{gillespie2007stochastic} to generate trajectories from a stochastic SIR model \cite{pajaro2022stochastic} using the parameters $\beta = 0.6, \alpha = 1/14$.  The population size was set to 200.  Trajectories each had 101 equispaced observations on the interval $[0,100]$.  Each dimension was scaled by the population size so that the trajectories represented proportions of the population.  There were 70 training trajectories, 50 validation trajectories, and 80 test trajectories.    The initial condition was given by a single infected individual with the remaining designated as susceptible.  Thus, for this dataset, $M >1$, and took the value 70 on the training set and 120 on the combined training and validation sets.  See figure \ref{fig:stochastic_SIR}.
\end{enumerate}

\begin{figure}
    \centering
    \includegraphics[width=0.5\linewidth]{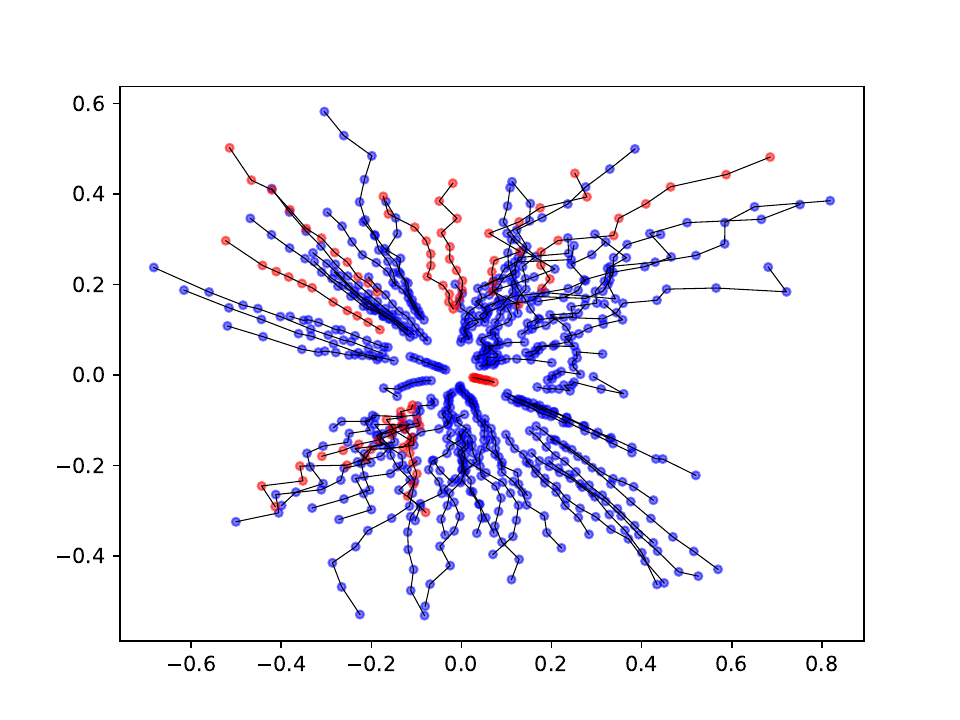}
    \caption{Dense matrix-valued diffusion dataset. Plot of train and validation sets.  Training data is in blue and validation data is in is red.  }
    \label{fig:train_val_data}
\end{figure}

\begin{figure}
  \begin{subfigure}{.49\linewidth}
    \includegraphics[width=1\linewidth]{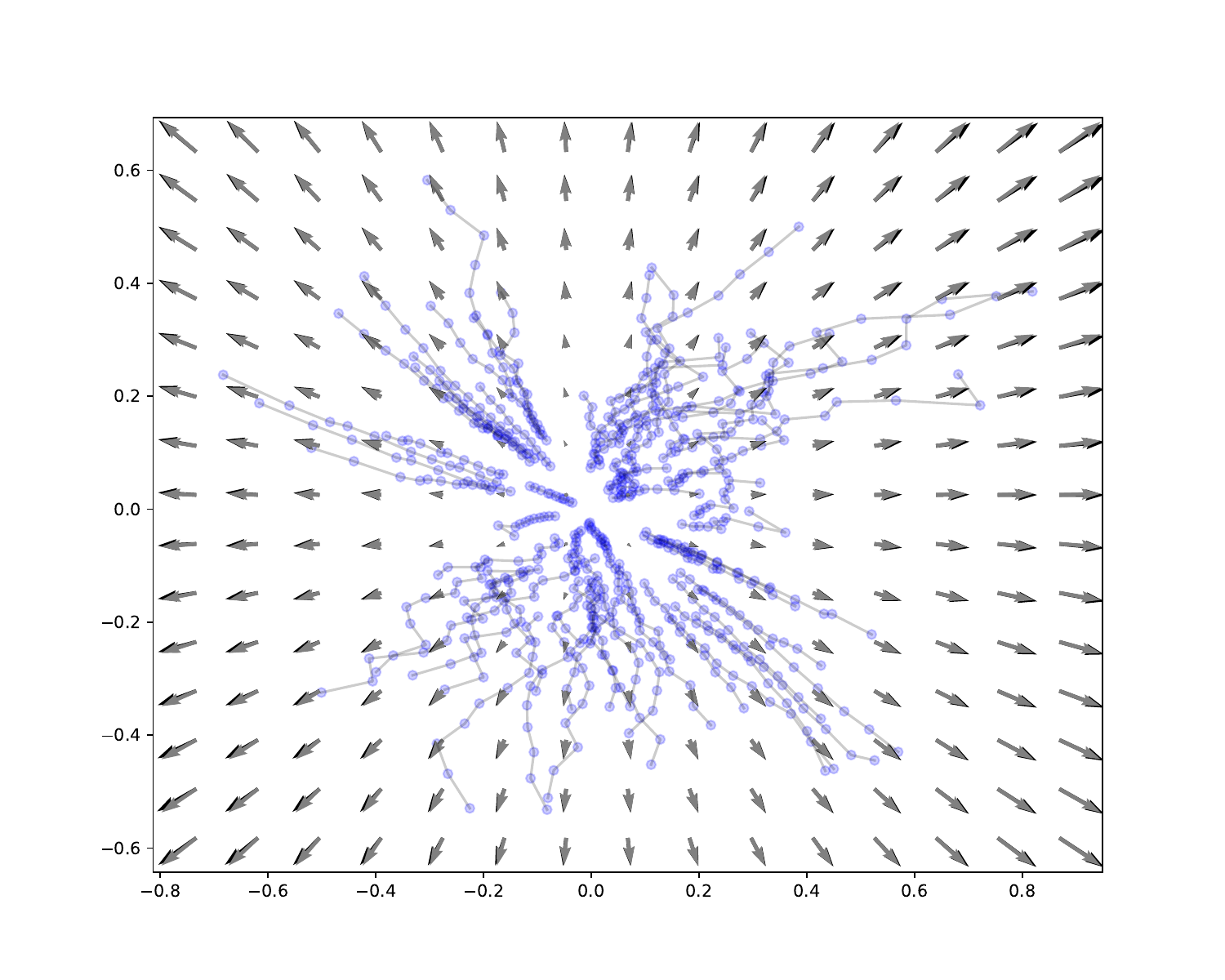}
   \caption{}
   \label{fig:drift_est}
   \end{subfigure}
   \begin{subfigure}{.49\linewidth}
       \includegraphics[width=1\linewidth]{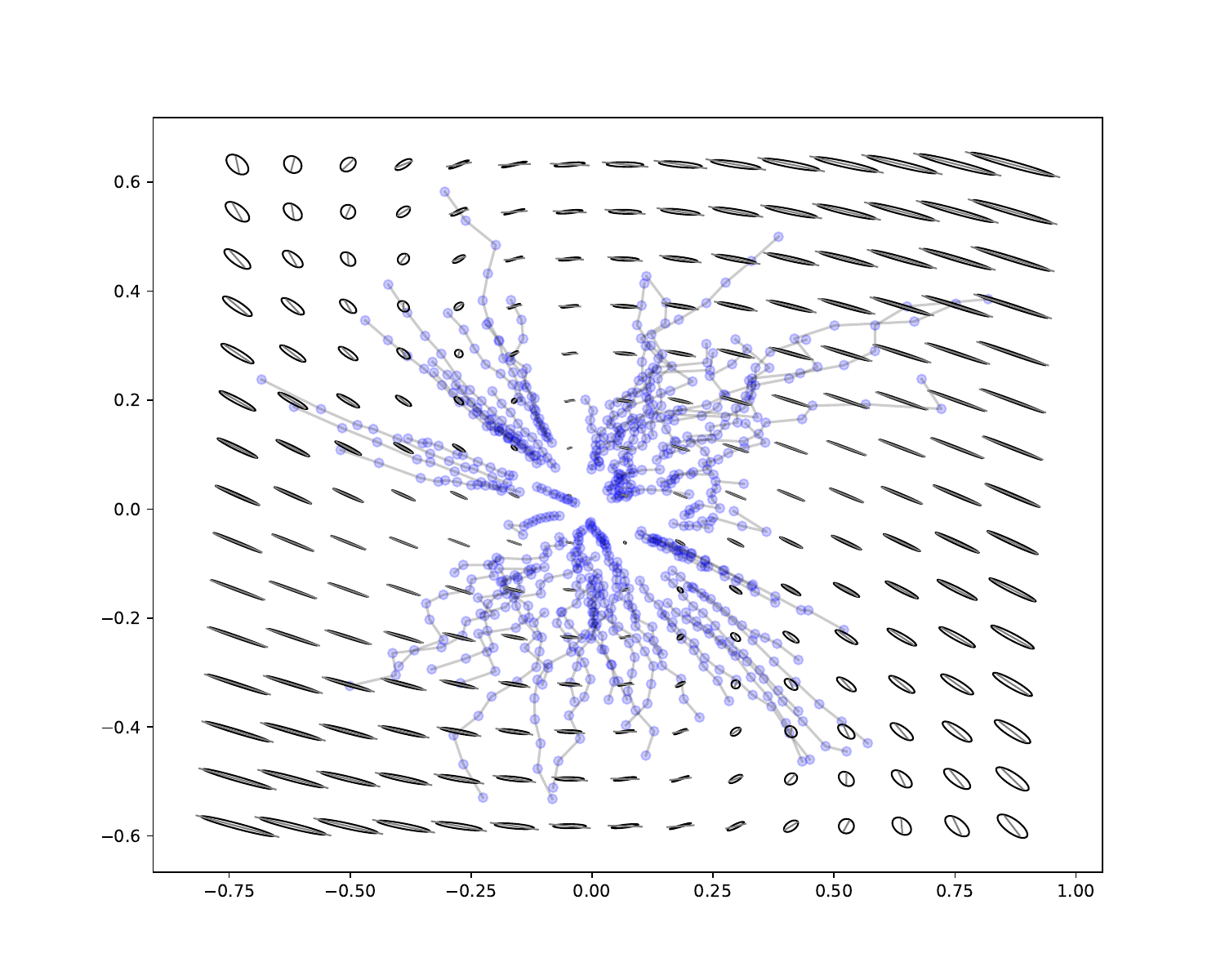}
       \caption{}
       \label{fig:diff_est}
   \end{subfigure}

    \caption{\small Dense matrix-valued diffusion dataset.  (\ref{fig:drift_est}) Plot of estimated drift in black and true drift in grey.  Training data is in the background.  (\ref{fig:diff_est})  Plot of estimated $a$ in black and true function $\sigma_0\sigma_0^T$ in grey.  The matrix-valued output is represented as ellipses.  Training data is in the background.}
    \label{fig:dataset_four}
\end{figure}
\begin{figure}
  \begin{subfigure}{.49\linewidth}
    \includegraphics[width=1\linewidth]{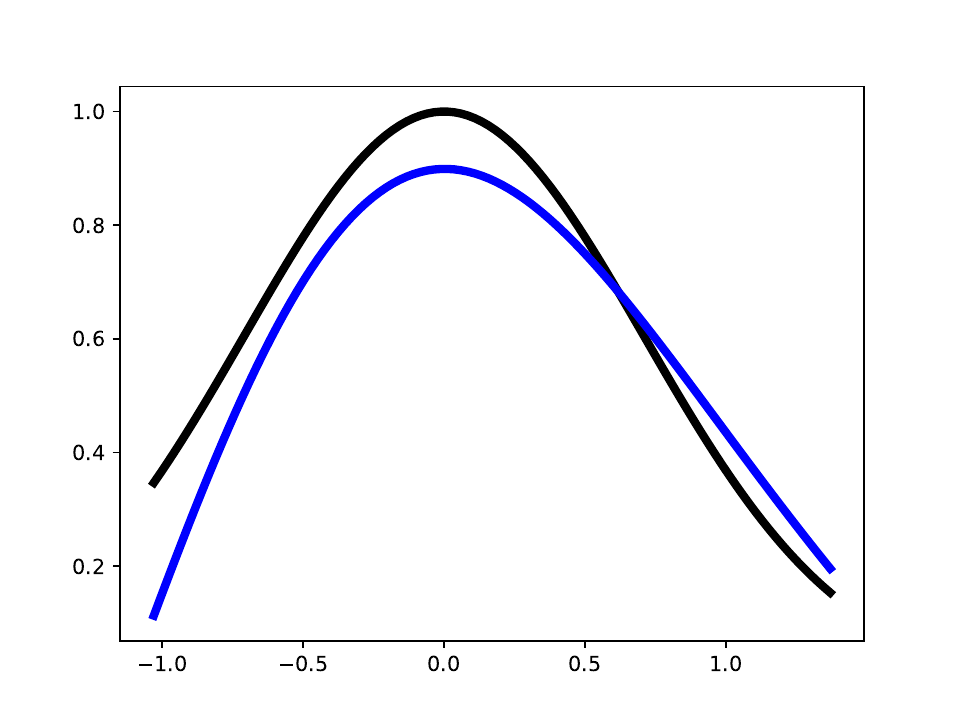}
   \caption{}
   \label{fig:drift_est_1D}
   \end{subfigure}
   \begin{subfigure}{.49\linewidth}
       \includegraphics[width=1\linewidth]{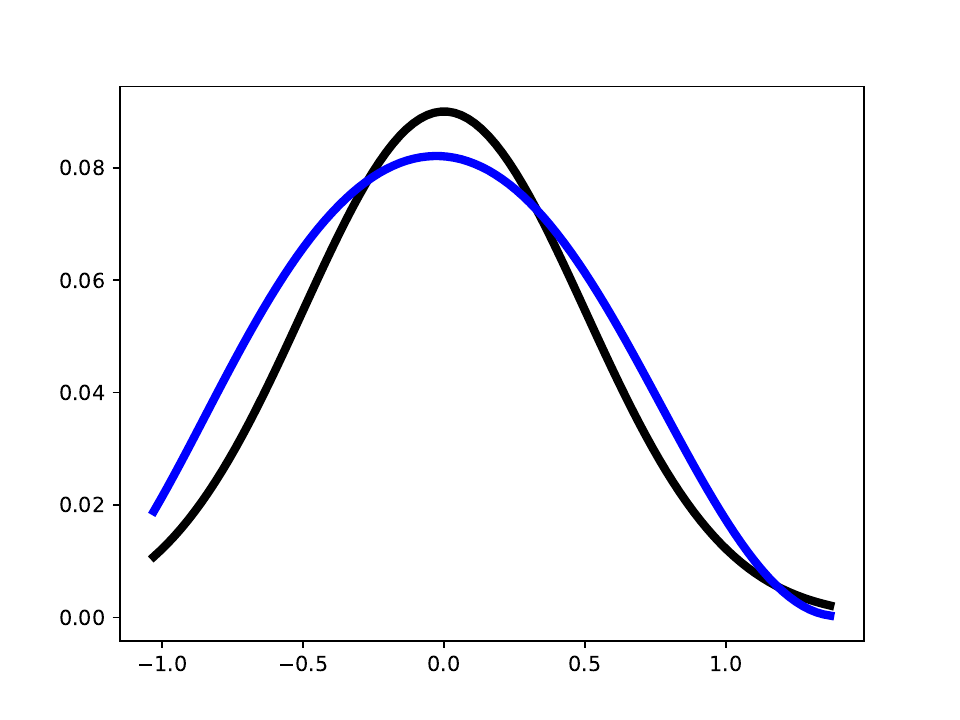}
       \caption{}
       \label{fig:diff_est_1D}
   \end{subfigure}

    \caption{\small Exponential dynamics dataset.  (\ref{fig:drift_est_1D}) Plot of estimated drift in blue and true drift in black.  (\ref{fig:diff_est_1D})  Plot of estimated $a$ in blue and true function $\sigma_0^2$ in black.  }
    \label{fig:exp_dynamics_est}
\end{figure}

We visualize the 2D dense matrix-valued diffusion
 dataset in figure \ref{fig:train_val_data}.  In figure \ref{fig:dataset_four}, we plot the true vector fields associated to the drift and diffusion along with a SOCK estimate on the dense matrix-valued diffusion dataset.  Since the function $\sigma_0\sigma_0^T$ is $(2,2)$ matrix-valued in this case, we plot its output as ellipses.  This is accomplished by taking an eigendecomposition of $\sigma_0\sigma_0^T(x)$ to obtain eigenvalues $\lambda_1 \leq \lambda_2$ and eigenvectors $v_1, v_2$.  We set $\theta:= \arctan(v_{2y}/v_{2x})$ and plot the ellipse centered at $x$, rotated by an angle of $\theta$, with height and width proportional to $\sqrt{\lambda_1}$,$\sqrt{\lambda_2}$, respectively.  Since the true function $\sigma_0\sigma_0^T$ has rank-one output, the grey ellipses appear as lines.
\subsection{Hyperparameter tuning}
\label{sec:hyperparameter_tuning}
We performed a grid search to find the optimal regularization parameters $\lambda_f, \lambda_\sigma$, as well the optimal kernel parameters (if any).  
\begin{enumerate}
    \item {\bf Approximate likelihood}:
To evaluate performance on the validation set, we computed an approximate likelihood in the following way.  First, we discretized the trajectories with the Euler-Maruyama scheme \cite{EulerMaruyama_MLE}, as given by  $y_{i+1} \approx y_i + hf(y_i) + \sqrt{h}\sigma(y_i)\epsilon_i$,  where $\epsilon_i \sim N(0, I)$.  This implies that the distribution of $y_{i+1}$ given $y_i$ is approximately $y_{i+1} \vert y_i \sim N\left(y_i + hf(y_i), ha(y_i)\right)$.  To reduce numerical instability, we introduced the quantity $\delta > 0$ and replaced $ha(y_i)$ with $ha(y_i) + \delta I_d$ \cite{DARCY2023133583}.  We used $\delta = 0.0$ on most of the datasets.  For the other datasets, we found $\delta = 10^{-3}$ was sufficient.  Using the Markov property of the solution of an SDE \cite{diffusions_markov_martingale_Rogers1987}, we get that
$p(y_0, \ldots, y_n \vert f, a) = \prod_{i=0}^{n-1}p(y_{i+1} \vert y_i, f, a)$, where we make the additional assumption that the initial point satisfied $p(y_0 \vert f, a)=  1$.  We select the hyperparameters that maximize this approximate likelihood on the validation set.
\end{enumerate}
\subsection{Evaluation metrics on simulated data} 
\label{sec:metrics_simulated_data}
\begin{enumerate}
    \item {\bf Perplexity}: We compute the log of the approximate likelihood on the test set as in section \ref{sec:hyperparameter_tuning}, followed by normalizing by the number of test points and dimension of the data.  We then take the exponential of the negative log likelihood and divide by the corresponding quantity for the true model, reporting the result as a percentage:  $P := \exp\left(\frac{l_{\rm true} - l_{\rm model}}{dn}\right)\cdot 100\%$.  The score represents the ratio (in percentage) of the estimated model's per-token, per-dimension perplexity to that of the true model. Perplexity is commonly used to evaluate probabilistic models by measuring how well the model predicts a sequence of observations \cite{jelinek1977perplexity}. A lower score indicates that the estimated model assigns a higher likelihood to the test data, relative to the true model.
    
    \item {\bf Relative error}:  Given an estimated function $\hat{g}$, we compute the matrix of function evaluations $\hat{G}$ on the data $\{y_i\}_{i=0}^n$ where $\hat{G}_i := \hat{g}(y_i)$.  We compute the analogous matrix $G$ for the true function.  The RE metric is then given as a percentage:  $RE_{\hat{g}} := \frac{\|\hat{G} - G\|_F}{\|G\|_F}\cdot 100\%$, where $\|\cdot\|_F$ is the Frobenius norm.  We report $RE_f, RE_{a}$ as evaluated on the test data.  A lower score indicates less point-wise error in the estimated functions as compared to the true functions.
    \item {\bf MMD statistic}:
    \label{sec:MMD_metric}
We also use the maximum mean discrepancy (MMD) \cite{JMLR:v13:gretton12a} to measure model performance.  We use a kernel defined over trajectories in $L^2([a,b], \mathbb{R}^d)$ given by the formula
\begin{equation}
    K_{MMD}(x, y) := \exp\left(-\frac{1}{\eta}\int_a^b \|x(t) - y(t)\|^2 dt\right)
    \label{eq:K_MMD}
\end{equation}
The parameter $\eta$ is set to be the median of all pairwise squared $L^2$- distances between trajectories in the training set \cite{garreau2018largesampleanalysismedian}.  This kernel is characteristic \cite[Theorem 9]{JMLR:v23:20-1180}, and thus it is an appropriate choice for the MMD.  We discretize the integral using a trapezoid-rule quadrature.  Trajectories are integrated using Euler-Maruyama with a step-size of $0.001$.  For the evaluation on the test set, we generate 500 trajectories for each distinct initial condition and compute the MMD using these and the true test trajectories.  This gives us a collection of MMD statistics (one for each initial condition in the test set).  We report the mean of these quantities.  In practice, we use an empirical MMD statistic with a discretized version of the kernel $K_{MMD}$.  A lower score indicates a greater similarity between the distributions of the model trajectories and test data.

\end{enumerate}
\subsection{Evaluation metrics on real data}
\label{sec:metrics_real_data}
\begin{enumerate}
    \item {\bf Perplexity}: Here we compute the perplexity as given by $P' := \exp\left(-\frac{l_{\rm model}}{dn}\right) \cdot 100\%$.

    \item {\bf MMD statistic}:  See section \ref{sec:metrics_simulated_data}.  
\end{enumerate}

\subsection{Kernels used}
\label{sec:kernels_used}
We used several different kernels with the SOCK method.  We describe them here along with any parameters they depend on:

\begin{enumerate}
    \item {\bf Gaussian kernel}:
    \begin{equation}
        K(x, y) = \exp\left(-\frac{\|x-y\|^2}{2\eta^2}\right)I_d
    \end{equation}
    where $\eta > 0$ is a scale parameter.  Denoted 'Gauss' in tables \ref{table:simulated_results} and \ref{table:real_results}. We found the optimal $\eta$ with a validation set.
    \item {\bf Polynomial kernel}:
    \begin{equation}
        K(x, y) = (x^T y + 1)^c I_d
    \end{equation}
    where $c \in \mathbb{N}$ is the polynomial degree.  We used $c=2$.  Denoted 'poly' in table \ref{table:simulated_results}.
    \item {\bf Linear kernel}:
    \begin{equation}
        K(x,y) = (x^T y) I_d
    \end{equation}
    Denoted 'linear' in table \ref{table:simulated_results}.
    \item {\bf Gaussian Fourier features}:
    \begin{equation}
        K(x,y) = (\varphi(x)^T \varphi(y)) I_d
    \end{equation}
    where
    \begin{equation}
        \varphi(x) := \sqrt{\frac{2}{p}}\begin{bmatrix}
            \cos(\omega_1^T x/\eta + \beta_1)\\
            \vdots\\
            \cos(\omega_p^T x/\eta + \beta_p)
        \end{bmatrix}
    \end{equation}
    with $\omega_i \sim N(0, I_d), \beta_i \sim \mbox{Unif}(0, 2\pi)$ for $i=1,\ldots p$.  We set $p := 100$ for all experiments.  The quantity $\eta > 0$ is a scale parameter.  We found $\eta$ with a validation set.  Denoted 'FF' in tables \ref{table:simulated_results} and \ref{table:real_results}.
\end{enumerate}
\subsection{Comparators}
\label{sec:comparators}
The methods we compare against in our numerical experiments are as follows:
\begin{enumerate}
    \item {\bf BISDE}:   This method employs a technique similar to SINDy \cite{SINDy} for ODEs.  First, finite-difference approximations are used to estimate evaluations of the drift and diffusion.  Next, a sparse regression is performed using a library of candidate functions to find estimates of the deterministic and stochastic components of the dynamics.  Rather than directly estimate the diffusion, the function $\sigma_0 \sigma_0^T$ is estimated instead.  However, there is no constraint imposed that this estimate be PSD-matrix-valued. \cite{BISDE2022244}
    \item {\bf gEDMD}: 
    Instead of estimating nonlinear dynamics on the observation space, this method estimates linear dynamics on the space of observables.  The generator of the Koopman semigroup is estimated from data, from which one then derives estimates of the drift and the function $\sigma_0\sigma_0^T$.  Like the previous method, there is no guarantee that the estimate of $\sigma_0\sigma_0^T$ is PSD-matrix-valued. \cite{gEDMD-Klus}
  
     \item {\bf Neural SDE}:
    We compared SOCK to each of the neural SDE methods described in \cite{neural-sde_stable_sde_Oh2024}.  They each define a model on latent space, relying on learned projections to map data to and from the hidden dimensions.  We denote the latent variable by $z$.  Each of the models has hyperparameters for the hidden space dimension and the number of layers in the neural network architectures used for the latent drift and diffusion. We validated these according to the MMD statistic of section \ref{sec:hyperparameter_tuning} as evaluated on the validation set.  We trained each model for 200 epochs using the mean-squared error (MSE) loss provided in the code repository.  Note that since neural SDE methods do not model the drift and diffusion directly, we are not able to compute the perplexity or relative error metrics for them.  We evaluate them solely with the MMD.
    
    The classical neural SDE model is given by
    \begin{equation}
        dz_t = f(t, z_t; \theta_f) dt + g(t, z_t; \theta_g) dW_t
    \end{equation}
    where $f$ and $g$ are neural networks with parameters $\theta_f, \theta_g$, respectively.
    \item {\bf Langevin-type SDE (Neural LSDE)}:
    This model is defined as:
    \begin{equation}
        dz_t = \gamma(z_t ; \theta_{\gamma}) dt + \sigma(t;\theta_{\sigma})dW_t
    \end{equation}
    Note that the drift depends only on the spatial variable while the diffusion depends only on time.
    \item {\bf Linear Noise SDE (Neural LNSDE)}:
    The model is defined as follows:
    \begin{equation}
        dz_t = \gamma(t, z_t ;\theta_{\gamma}) dt + \sigma(t;\theta_{\sigma})z_t dW_t
    \end{equation}
    Here the diffusion is a time-dependent linear function of the latent variable $z_t$.
    \item {\bf Geometric SDE (Neural GSDE)}:
    This model is given by:
    \begin{equation}
        \frac{dz_t}{z_t} = \gamma(t, z_t ; \theta_{\gamma}) dt + \sigma(t; \theta_{\sigma})dW_t 
    \end{equation}
    Note that $z_t$ will be an exponential function in this case.
\end{enumerate}

The NAs for the perplexity metric in tables \ref{table:simulated_results} and \ref{table:real_results} associated to the BISDE and gEDMD methods occur because the estimates of $\sigma_0 \sigma_0^T$ that these algorithms obtain are not necessarily PSD matrix-valued.  This invalidates the average likelihood and thus precludes the computation of the perplexity metric.  The NAs for the MMD metric occurred when there was numerical instability in the SDE models' solutions.  There are NAs for the neural SDE methods for all metrics besides the MMD.  As mentioned in section \ref{sec:comparators}, neural SDE methods do not model the drift and diffusion of the SDE under investigation; thus, we are limited in our ability to evaluate them.  We rely solely on the MMD metric using trajectories predicted by the model for this collection of methods.

As demonstrated in tables \ref{table:simulated_results} and \ref{table:real_results}, the SOCK method performs competitively with state-of-the-art methods for learning the dynamics of SDEs.  We demonstrate superior capabilities for estimating the function $\sigma_0\sigma_0^T$ in all cases and our estimates of $f_0$ are frequently the best.  Our method also obtained the best results on the real, medical dataset and the stochastic SIR dataset.

The BISDE method obtained good estimates for the drift and was frequently competitive in the MMD metric.  The gEDMD method did well on low-dimensional data, but not high-dimensional data.  The neural SDE methods did not achieve the best results on any dataset.  We speculate that this is because we used the MSE loss for training them, and thus, the trained models were not optimal for the MMD metric.  The MMD loss might have been more appropriate for training; however, this loss is numerically unstable and difficult to work with.  We frequently obtained NaN's using this loss, for example.

\begin{table}

  \small
  \caption{Results on simulated data.   The default SOCK method is SOCK Gauss/FF.  The SOCK method with adapted kernel appears in italics.  Best results are marked with $^*$.  Best results besides the SOCK adapted kernel method are marked with $\dagger$.  NA appears when the evaluation metric could not be computed.   See section \ref{sec:kernels_used} for details on kernels used.}
  \begin{subtable}{1.0\linewidth}


    \centering
    \begin{tabular}{l l l l l}
        \toprule
        \multicolumn{5}{c}{Geometric Brownian motion}\\
        Method       & $P$       & $RE_f$  & $RE_{a} $ & MMD\\
        \midrule
        {\em SOCK linear/linear}       & {\em 102.7}$^*$      & {\em 4.878e-2}$^*$  & {\em 12.66} & {\em 9.366e-2}\\
       
        SOCK Gauss/FF &114.6$\dagger$ & 10.33$\dagger$& 8.740$^*$$\dagger$ & 8.191e-2$^*$$\dagger$\\
        BISDE &  NA & 33.62 & 65.47 & 8.354e-2\\
        gEDMD & 117.2 & 19.13 &  32.68 & 8.465e-2\\
        Neural SDE & NA & NA & NA & 1.945e-1\\
        Neural LSDE & NA & NA & NA & 1.948e-1\\
        Neural LNSDE & NA & NA & NA & 1.969e-1\\
        Neural GSDE & NA & NA & NA & 2.041e-1\\
        \bottomrule
    \end{tabular}%
  \end{subtable}
   \begin{subtable}{1.0 \linewidth}
       \centering
    \begin{tabular}{l l l l l}
        \toprule
        \multicolumn{5}{c}{Exponential dynamics}\\
        Method       & $P$       & $RE_f$  & $RE_{a}$ & MMD \\
        \midrule
        SOCK Gauss/FF & 100.9$^*$$\dagger$ &12.42&  19.15$^*$$\dagger$ & 1.196e-1\\
        BISDE & 103.4 & 6.510$^*$$\dagger$ & 34.67 & 1.350e-1\\
        gEDMD & 105.9 & 14.13& 20.31 & 1.186e-1$^*$$\dagger$\\
        Neural SDE & NA & NA & NA & 1.386e-1\\
        Neural LSDE & NA & NA & NA & 1.504e-1\\
        Neural LNSDE & NA & NA & NA & 1.328e-1\\
        Neural GSDE & NA & NA & NA & 1.420e-1\\
        \bottomrule
    \end{tabular}%
   \end{subtable}
   \begin{subtable}{1.0 \linewidth}
   \centering
   \begin{tabular}{l l l l l}
        \toprule
        \multicolumn{5}{c}{Dense matrix-valued diffusion}\\
        Method       & $P$      & $RE_f$  & $RE_{a}$ & MMD \\
        \midrule
        {\em SOCK linear/linear} & {\em 99.84}$^*$ & {\em 8.192}$^*$ & {\em 17.22 }$^*$ & {\em 5.550e-2}$^*$\\
        SOCK Gauss/FF & 101.5$\dagger$ & 22.30 & 52.70$\dagger$  & 6.522e-2\\
        BISDE & 101.5 & 12.54 & 60.82 & 6.440e-2$\dagger$\\
        gEDMD & NA & 8.634$\dagger$ & 398.1 & 6.598e-2\\
        Neural SDE & NA & NA & NA &1.786e-1 \\
        Neural LSDE & NA & NA & NA & 1.932e-1\\
        Neural LNSDE & NA & NA & NA & 1.923e-1\\
        Neural GSDE & NA & NA & NA &1.800e-1 \\
        \bottomrule
    \end{tabular}%
    \end{subtable}
     \begin{subtable}{1.0 \linewidth}
     \centering
    \begin{tabular}{l l l l l}
        \toprule
        \multicolumn{5}{c}{Stochastic Lorenz 96-10}\\
        Method       & $P$      & $RE_f$  & $RE_{a}$ & MMD\\
        \midrule
        {\em SOCK poly/linear} & {\em 103.3}$^*$ & {\em 26.84} & {\em 19.06 }$^*$ & NA\\
        SOCK Gauss/FF & 104.3$\dagger$ & 26.11$^*$$\dagger$ & 26.16$\dagger$ & 6.918e-1 $^*$$\dagger$\\
        BISDE & NA & 46.46 & 122.2 & 7.191e-1\\
        gEDMD & NA & 89.96 & 6757 & NA\\
        Neural SDE & NA & NA & NA & 9.947e-1\\
        Neural LSDE & NA & NA & NA & 9.875e-1\\
        Neural LNSDE & NA & NA & NA & 9.874e-1\\
        Neural GSDE & NA & NA & NA & 9.921e-1\\
        \bottomrule
    \end{tabular}%
     \end{subtable}
     \label{table:simulated_results}
\end{table}

\begin{table}[h!]

  \small
 
  \caption{Results on amyloid and stochastic SIR datasets.  Best results are marked with $^*$.  The default SOCK method is SOCK Gauss/FF.  See section \ref{sec:kernels_used} for details on kernels used.}
  
    \begin{subtable}{0.49\linewidth}
   
    \begin{tabular}{l l l}
        \toprule
        \multicolumn{3}{c}{Amyloid data}\\
        Method       & $P'$ & MMD\\
        \midrule
        SOCK Gauss/FF&   868.5$^*$ &  2.459e-1$^*$\\
        BISDE & 1.204e5 & 4.912e-1 \\
        gEDMD & NA & 2.492e-1\\
        Neural SDE & NA & 3.625e-1\\
        Neural LSDE & NA & 3.700e-1\\
        Neural LNSDE & NA & 3.576e-1\\
        Neural GSDE& NA & NA\\
        \bottomrule
    \end{tabular}%
  \end{subtable}
  \begin{subtable}{0.49\linewidth}
      \begin{tabular}{l l l}
      \toprule
      \multicolumn{3}{c}{Stochastic SIR}\\
       Method    & $P'$ & MMD  \\
       \midrule
        SOCK poly/linear  & 10.15$^*$ & 6.733e-2$^*$\\
        SOCK Gauss/FF & 19.47 & 2.412e-1\\
        BISDE & NA & 3.783e-1\\
        gEDMD & NA & 3.779e-1\\
        Neural SDE & NA & 3.020e-1\\
        Neural LSDE & NA & 1.213\\
        Neural LNSDE & NA & 1.223\\
        Neural GSDE & NA & 1.214\\
        \bottomrule
      \end{tabular}
  \end{subtable}
 
  \label{table:real_results}
  \end{table}
\section{Conclusion}
\label{seC:conclusion}
We have adapted the occupation kernel method to the stochastic setting by introducing expectations over trajectories with fixed initial conditions in the definition of occupation kernels.  Using the It\^{o} isometry, we were able to derive a simple reconstruction-error-based loss function for the diffusion.  We introduced the operator-valued occupation kernel to optimize this loss efficiently.  We used the parametrization of \cite{nonnegative_funcs_marteau2020}, \cite{learning_psd_valued_Muzellec}, adapting the representer theorem for PSD matrix-valued functions to utilize occupation kernels, which amounted to incorporating integral functionals in the place of evaluation operators.   In this way, we have constrained our estimate $\sigma_0\sigma_0^T$ to be PSD-matrix-valued, in accordance with the true function.  We have demonstrated competitive results on several simulated datasets as well as a real dataset and a stochastic SIR dataset.  
\subsection{Discussion}

The regularity conditions in theorem \ref{thm:drift_representer_thm} and proposition \ref{prop:operator_occ_kerns} are not overly restrictive.  For example, using a bounded, universal kernel will be sufficient for all of them.  In the explicit kernel case, condition ii) of proposition \ref{prop:operator_occ_kerns} is equivalent to the range of the feature map spanning its codomain.  When we use a polynomial kernel of degree $c$, the conditions on the boundedness of the expected integral of the kernel are equivalent to the process $x_t$ having finite moments up to order $c$.  Thus, for most settings that one would encounter in practice, the regularity conditions are satisfied.  

The choice of kernel greatly affects the model's performance, as seen in table \ref{table:simulated_results}.  One may use a validation set to select an appropriate kernel in practice.  The Gaussian kernel for the drift and Gaussian Fourier features for the diffusion serve as our default kernel choices because they are effective on a wide variety of datasets.
\subsection{Limitations}
\label{sec:limitations}
We have assumed that the training data does not contain additive white Gaussian noise.  The derivation of the diffusion method will have to be modified in that case.

The Fenchel duality used in the optimization of the diffusion may be unnecessary.  It may be more prudent to use a standard method like projected gradient descent.  At this point, it is unclear which is the superior method.  An ablation study is needed.

 \subsection{Future work}
 We plan to enhance SOCK by introducing an iterative framework of generating latent trajectories conditioned on the observations followed by re-estimating the dynamics, repeating until convergence.  We hypothesize this will make the method more robust to sparsity and noise.


\section*{Acknowledgments}
We would like to thank Dr. Victor Rielly for his insight into the use of projected gradient descent as a viable optimization technique for the diffusion.

\bibliographystyle{siamplain}
\bibliography{references}
\appendix
\section{Proofs from section \ref{sec:est_drift}}
\label{sec:drift_derivation}
Following \cite{rielly2025mockalgorithmlearningnonparametric}, we define operators $L_i \colon H \to \mathbb{R}^d$ by
\begin{equation}
    L_i(f) := \mathbb{E}\left[\int_{t_i}^{t_{i+1}} f(x_t) dt \Bigg \vert x_0 \right]
\end{equation}
for $i=0,\ldots,n-1$.  Thus, for $v \in \mathbb{R}^d$, the function $f \mapsto L_i(f)^T v$ is a linear functional on $H$.  Under mild conditions on the kernel $K$, we may show that these functionals are bounded for each $v$.   
\begin{proposition}

A sufficient condition for $f \mapsto L_i(f)^T v$ to be bounded for each $v \in \mathbb{R}^d$ is
\begin{equation}
    \mathbb{E}\left[\int_{t_i}^{t_{i+1}}\mbox{\rm Tr}(K(x_t, x_t)) dt \Bigg \vert x_0\right] < \infty
\end{equation}
where $\mbox{Tr}(\cdot)$ indicates the trace of a matrix.
\end{proposition}
\begin{proof}
    We have
    \begin{align}
        \| L_i(f) \|_{\mathbb{R}^d}^2 &= \left \|\mathbb{E}\left[\int_{t_i}^{t_{i+1}}f(x_t) dt \Bigg \vert x_0 \right]\right\|^2_{\mathbb{R}^d}\\
        &\leq \mathbb{E}\left[\left\|\int_{t_i}^{t_{i+1}}f(x_t) dt \right\|^2_{\mathbb{R}^d}\Bigg \vert x_0 \right]\\
        &\leq \mathbb{E}\left[\int_{t_i}^{t_{i+1}}\|f(x_t)\|_{\mathbb{R}^d}^2dt\Bigg \vert x_0\right](t_{i+1} - t_i)\\
        &= \mathbb{E}\left[\int_{t_i}^{t_{i+1}}\sum_{k=1}^d \langle f, K(\cdot, x_t)e_k \rangle_H^2 dt \Bigg \vert x_0\right](t_{i+1} - t_i)\\
        &\leq \|f\|_H^2 \mathbb{E}\left[\int_{t_i}^{t_{i+1}}\sum_{k=1}^d \|K(\cdot, x_t)e_k\|^2_H dt\Bigg \vert x_0 \right](t_{i+1} - t_i)\\
        &= \|f\|^2_H \mathbb{E}\left[\int_{t_i}^{t_{i+1}}\sum_{k=1}^d e_k^T K(x_t ,x_t)e_k dt \Bigg \vert x_0 \right](t_{i+1} - t_i)\\
        &= \|f\|^2_H \mathbb{E}\left[\int_{t_i}^{t_{i+1}}\mbox{\rm Tr}(K(x_t, x_t)) dt \Bigg \vert x_0 \right](t_{i+1} - t_i)
    \end{align}
  where we used Jensen's inequality in the second line, Cauchy-Schwarz in the third, the reproducing property of $H$ in the fourth, and Cauchy-Schwarz again in the fifth.  Thus, if
    \begin{equation}
        \mathbb{E}\left[\int_{t_i}^{t_{i+1}}\mbox{\rm Tr}(K(x_t, x_t))dt\Bigg \vert x_0 \right] < \infty
    \end{equation}
    then $L_i$ is bounded.  The result follows.
\end{proof}
 For example, if the kernel $K$ is given by
\begin{equation}
    K(x,y) = g(x-y)
\end{equation}
for some function $g \colon \mathbb{R}^d \to \mathbb{R}^{d \times d}$, then the operators $L_i$ are bounded.  Gaussian kernels satisfy this property.

By the Riesz representation theorem, for each $v \in \mathbb{R}^d$, there exists a function $\ell_{i,v}^* \in H$ such that
\begin{equation}
    L_i(f)^T v = \langle \ell_{i, v}^*, f \rangle_H
\end{equation}
The functions $\ell_{i, v}^*$ are referred to as {\em occupation kernels} \cite{rielly2025mockalgorithmlearningnonparametric}\cite{Rosenfeld_2019}\cite{rosenfeld2024}. We may define functions $L_i^* \colon \mathbb{R}^d \to \mathbb{R}^{d \times d}$ by
\begin{equation}
    L_i^*(x) := \mathbb{E}\left[\int_{t_i}^{t_{i+1}}K(x, x_t) dt \Bigg \vert x_0\right]
\end{equation}
for $i=0,\ldots,n-1$.  They have the property that
\begin{equation}
    L_i(f)^T v = \langle L_i^* v, f \rangle_H
\end{equation}
for all $v \in \mathbb{R}^d$.  This permits us to write
\begin{equation}
    J_{drift}(f) = \frac{1}{n}\sum_{i=0}^{n-1}\left\|[\langle L_i^* e_j, f \rangle_H]_{j=1}^d - \overline{y_{i+1} - y_i}\right\|^2_{\mathbb{R}^d} + \lambda_f \|f\|^2_H
\end{equation}
 
{\em Proof of \ref{theorem:representer_thm}.}:
We prove the case where $K$ is $I$-separable.  The general case is similar.

 Since the functional $f \mapsto L_i(f)^T e_l$ is bounded, there exists a function $\ell_i^*$ such that
 \begin{equation}
     L_i(f)^T e_l = \langle \ell_i^*, f_l \rangle_{H_1}
 \end{equation}
 where $H_1$ is the RKHS of scalar-valued functions with kernel $k$ and $f_l(x) := f(x)^T e_l$.  In fact, the function $\ell_i^*$ is given by
 \begin{equation}
     \ell_i^*(x) = \mathbb{E}\left[\int_{t_i}^{t_{i+1}} k(x, x_t) dt \Bigg \vert x_0\right]
 \end{equation}
Let $V$ be the finite-dimensional subspace of $H_1$ spanned by $\ell_i^*$ for $i=0,\ldots, n-1$.  Using orthogonal projection, we may decompose $f_l$ as
 \begin{equation}
     f_l = f_{l}^V + f_{l}^{V^\perp}
 \end{equation}
 where $f_{l}^V \in V$ and $f_{l}^{V^\perp} \in V^\perp$.  We get that
 \begin{equation}
     \langle \ell_i^*, f_l \rangle_{H_1} = \langle \ell_i^*, f_{l}^V + f_{l}^{V^\perp} \rangle_{H_1} = \langle \ell_i^*, f_l^V \rangle_{H_1}
     \label{eq:proj_f}
 \end{equation}
 We let $\mathbf{V} := V \times \ldots \times V$ ($d$ times) and define $\Pi_{\mathbf{V}} \colon H \to \mathbf{V}$ as orthogonal projection onto $\mathbf{V}$.  Likewise we let $\Pi_{\mathbf{V}^\perp} \colon H \to \mathbf{V}^\perp$ be orthogonal projection onto $\mathbf{V}^\perp$.  Then
 \begin{align}
     \|f\|^2_H &= \|\Pi_{\mathbf{V}}f + \Pi_{\mathbf{V}^\perp}f\|^2_H\\
     &= \|\Pi_{\mathbf{V}}f\|^2_H + \|\Pi_{\mathbf{V}^\perp}f\|^2_H\\
     &\geq \|\Pi_{\mathbf{V}}f\|^2_H
 \end{align}
 These two facts together imply that for any $f \in H$, we have $J_{drift}(\Pi_{\mathbf{V}} f) \leq J_{drift}(f)$.  Thus, there exists $f \in \mathbf{V}$ that minimizes \eqref{eq:drift_cost}.  Then for each $l=1,\ldots,d$, we have
 \begin{equation}
     f_l^* = \sum_{i=0}^{n-1}\ell_i^* \alpha_i^{*l}
 \end{equation}
 for some $\alpha^{*l} \in \mathbb{R}^n$.  This implies
 \begin{equation}
     f^* = \sum_{i=0}^{n-1}L_i^* \alpha_i^*
 \end{equation}
 for an $\alpha^* \in \mathbb{R}^{n \times d}$.  Substituting this expression for $f^*$ into \eqref{eq:drift_cost}, we get
 \begin{equation}
     J(\alpha) = \frac{1}{n}\sum_{i=0}^{n-1}\sum_{l=1}^d\left\{ \left(\left \langle \ell_i^*, \sum_{j=0}^{n-1} \ell_j^* \alpha_j^l \right\rangle_{H_1} - \overline{y^l_{i+1} - y^l_i}\right)^2 + \lambda_f \alpha^{lT}L^* \alpha^l\right\}
 \end{equation}
 where $L^* \in \mathbb{R}^{n \times n}$ is the matrix given by
 \begin{equation}
     [L^*]_{lm} := \mathbf{E}_{\omega_1 \in \Omega, \omega_2 \in \Omega}\left[\int_{t_l}^{t_{l+1}}\int_{t_m}^{t_{m+1}}k(x_s(\omega_1), x_t(\omega_2)) dt ds \Bigg \vert x_0 \right ]
 \end{equation}
 We may write this compactly as
 \begin{equation}
     J(\alpha) = \frac{1}{n}\|L^* \alpha - \overline{\Delta y}\|^2_{\mathbb{R}^{n \times d}} + \lambda_f \mbox{Tr}(\alpha^T L^* \alpha)
 \end{equation}
 where $\overline{\Delta y} \in \mathbb{R}^{n \times d}$ is the concatenation of $\overline{y_{i+1} - y_i}$.  It follows that we may find the optimal $\alpha^*$ by solving
 \begin{equation}
     (L^* + n\lambda_f I_n)\alpha^* = \overline{\Delta y}
 \end{equation}
$\hfill \Box$
 
 \section{Explicit kernel diffusion estimation method}
\label{sec:derivation_outer_diff}
Here, we assume that $a(x)$ is given by
\begin{equation}
    a(x) := \hat{\varphi}(x)^T Q \hat{\varphi}(x), \;\;\; Q \succeq 0, Q^T = Q
\end{equation}
where $\hat{\varphi}(x) := I_d \otimes \varphi(x) \in \mathbb{R}^{pd \times d}$ with $\varphi(x)$ the $p$-dimensional feature map for the explicit, scalar-valued kernel $k(x, y) := \varphi(x)^T \varphi(y)$ and $Q$ is a $(pd, pd)$ matrix with $Q^T = Q$.  Let $H$ be the RKHS of scalar-valued functions corresponding to $k$. 

\begin{theorem}
\label{thm:parametrization_for_diff_explicit}
    The function $a (x) = \hat{\varphi}(x)^T Q \hat{\varphi}(x)$ is a $(d, d)$ matrix-valued function where each component $[a]_{kl}$ is in the RKHS $H'$ with kernel given by $k(x,y)^2 = (\varphi(x)^T \varphi(y))^2$ provided that the following regularity condition holds:
    \begin{equation}
        S^{p} = \overline{\mbox{\rm span}\{\varphi(x)\varphi(x)^T \colon x \in \mathbb{R}^d\}}
    \end{equation}
    where $S^p$ is the space of symmeric $(p,p)$ matrices.  In this case, $H'$ and the space of functions $x \mapsto \varphi(x)^T Q' \varphi(x)$ for $Q' \in S^p$ coincide.
\end{theorem}
\begin{proof}
    We may write
    \begin{equation}
        \hat{\varphi}(x)^T Q \hat{\varphi}(x) = \left[\varphi(x)^T Q_{kl} \varphi(x)\right]_{k,l=1}^d
    \end{equation}
    where $Q_{kl}$ is the $(k,l)$ block of $Q$ of size $(p,p)$.  Thus, we must show that $x \mapsto \varphi(x)^T Q' \varphi(x) \in H'$ for any $Q^{'T} = Q'$.  Suppose that $g(x) := \varphi(x)^T Q' \varphi(x)$.  Using the regularity condition, we get
    \begin{equation}
        Q' = \sum_{j=1}^\infty \varphi(x_j)\varphi(x_j)^T\beta_j
    \end{equation}
    for some countable sequences $\{x_j\}_{j=1}^\infty \subset \mathbb{R}^d $ and $\{\beta_j\}_{j=1}^\infty \subset \mathbb{R}$  .  Thus,
    \begin{align}
        \varphi(x)^T Q' \varphi(x) &= \varphi(x)^T \left( \sum_{j=1}^\infty \varphi(x_j)\varphi(x_j)^T \beta_j\right) \varphi(x)\\
        &= \sum_{j=1}^\infty (\varphi(x)^T \varphi(x_j))^2 \beta_j\\
        &= \sum_{j=1}^\infty k(x, x_j)^2 \beta_j
    \end{align}
    which shows that $g \in H'$.

    Conversely, suppose that $g \in H'$.  Then there exist $\{x_j\}_{j=1}^\infty \subset \mathbb{R}^d$ and $\{\beta_j\}_{j=1}^\infty \subset \mathbb{R}$ such that
    \begin{equation}
        g(x) = \sum_{j=1}^\infty k(x, x_j)^2 \beta_j
    \end{equation}
    Thus,
    \begin{align}
        \sum_{j=1}^\infty k(x, x_j)^2 \beta_j &= \sum_{j=1}^\infty (\varphi(x)^T \varphi(x_j))^2 \beta_j\\
        &= \varphi(x)^T\left[\sum_{j=1}^\infty \varphi(x_j)\varphi(x_j)^T \beta_j \right]\varphi(x)\\
        &= \varphi(x)^T Q' \varphi(x)
    \end{align}
    where $Q' := \sum_{j=1}^\infty \varphi(x_j)\varphi(x_j)^T \beta_j$.
\end{proof}
We denote the space of $(d,d)$ matrix-valued functions with entries in $H'$ by ${H'}^{d\times d}$.  We define the norm on ${H'}^{d\times d}$ by $\|a\|^2_{{H'}^{d \times d}} := \sum_{k,l=1}^d \|a_{kl}\|^2_{H'}$.  Note that ${H'}^{d\times d}$ is itself an RKHS with a tensor-valued kernel, but we do not use this fact.  When $Q \succeq 0$, we get that
\begin{equation}
    a(x) \succeq 0
\end{equation}
for all $x \in \mathbb{R}^d$.  This condition on the output of $a$ is obviously satisfied by the true function $\sigma_0\sigma_0^T$. 

\begin{proposition}
\label{prop:matrix_of_inner_products}
Suppose that the following regularity condition holds:
\begin{equation}
    \mathbb{E}\left[\int_{t_j}^{t_{j+1}}\|\varphi(x_t)\|^4_{\mathbb{R}^p} dt \Bigg \vert x_0 \right] < \infty, \;\; j=0,\ldots, n-1
\end{equation}
Then we may write
\begin{align*}
    \mathbb{E}\left[\int_{t_j}^{t_{j+1}} a(x_t) dt \Bigg \vert x_0 \right] &= \mathbb{E}\left[\int_{t_j}^{t_{j+1}} \hat{\varphi}(x_t)^T Q\hat{\varphi}(x_t) dt \Bigg \vert x_0 \right]\\
    &= \left[\left\langle \mathbb{E}\left[\int_{t_j}^{t_{j+1}} \varphi(x_t) \varphi(x_t)^T dt \Bigg \vert x_0 \right], Q_{kl} \right\rangle_{\mathbb{R}^{p \times p}}\right]_{k,l=1}^d
\end{align*}
for $j=0, \ldots, n-1$ where $Q_{kl}$ is the $(k,l)$-block of $Q$ of size $(p,p)$. 
\end{proposition}
\begin{proof}
By theorem \ref{thm:parametrization_for_diff_explicit}, we have that $[a]_{kl} \in H'$ for $k,l=1,\ldots, d$.  Let us define $P_j \colon H' \to \mathbb{R}$ to be the functional
\begin{equation}
    P_j(g) := \mathbb{E}\left[\int_{t_j}^{t_{j+1}} g(x_t) dt\Bigg \vert x_0 \right]
\end{equation}
for $j=0,\ldots, n-1$.  Thus,
\begin{align}
    \lvert P_j(g)\rvert^2 &= \left \lvert\mathbb{E}\left[\int_{t_j}^{t_{j+1}}g(x_t) dt \Bigg \vert x_0 \right]\right\rvert^2\\
    &\leq \mathbb{E}\left[\left\lvert\int_{t_j}^{t_{j+1}}g(x_t) dt \right\rvert^2\Bigg \vert x_0 \right]\\
    &\leq \mathbb{E}\left[\int_{t_j}^{t_{j+1}}\lvert g(x_t)\rvert^2 dt \Bigg \vert x_0\right](t_{j+1} - t_j)\\
    &= \mathbb{E}\left[\int_{t_j}^{t_{j+1}}\langle g, k(\cdot, x_t)^2 \rangle_{H'}^2 dt \Bigg \vert x_0 \right](t_{j+1} - t_j)\\
    &\leq \|g\|^2_{H'} \mathbb{E}\left[\int_{t_j}^{t_{j+1}} k(x_t,x_t)^2 dt \Bigg \vert x_0\right](t_{j+1} - t_j)\\
    &= \|g\|^2_{H'} \mathbb{E}\left[\int_{t_j}^{t_{j+1}}\|\varphi(x_t)\|^4_{\mathbb{R}^p} dt \Bigg \vert x_0\right](t_{j+1} - t_j)
\end{align}
where we have used Jensen's inequality on the second line, Cauchy-Schwarz on the third, the reproducing property on the fourth, and Cauchy-Schwarz again on the fifth.  Thus, by the regularity condition, the operator $P_j$ is bounded.  Suppose $g \in H'$.  By the Riesz representation theorem, we may write
\begin{equation}
    P_j(g) = \langle g, \gamma_i^* \rangle_{H'}
\end{equation}
for some $\gamma_i^* \in H'$.  A computation shows that
\begin{equation}
\gamma_i^*(x) = \mathbb{E}\left[\int_{t_j}^{t_{j+1}}(\varphi(x)^T \varphi(x_t))^2 dt \Bigg \vert x_0\right]
\end{equation}
Thus, we get that
\begin{equation}
    \mathbb{E}\left[\int_{t_j}^{t_{j+1}}(\varphi(x)^T\varphi(x_t))^2 dt \Bigg \vert x_0\right] = \varphi(x)^T \left(\mathbb{E}\left[\int_{t_j}^{t_{j+1}}\varphi(x_t) \varphi(x_t)^T dt \Bigg \vert x_0 \right]\right)\varphi(x)
\end{equation}
which shows that the matrix
\begin{equation}
    \mathbb{E}\left[\int_{t_j}^{t_{j+1}}\varphi(x_t)\varphi(x_t)^T dt \Bigg \vert x_0 \right] \in \mathbb{R}^{p \times p}
\end{equation}
Therefore, we may compute the following:
    \begin{align}
        \mathbb{E}\left[\int_{t_j}^{t_{j+1}} a(x_t) dt \Bigg \vert x_0 \right] &= \mathbb{E}\left[\int_{t_j}^{t_{j+1}}\hat{\varphi}(x_t)^T Q \hat{\varphi}(x_t) dt \Bigg \vert x_0\right]\\
        &= \mathbb{E}\left[\int_{t_j}^{t_{j+1}} \left[\varphi(x_t)^T Q_{kl} \varphi(x_t)\right]_{k,l=1}^d dt \Bigg \vert x_0\right]\\
        &= \mathbb{E}\left[\int_{t_j}^{t_{j+1}}\left[\mbox{Tr}(\varphi(x_t)\varphi(x_t)^T Q_{kl}) \right]_{k,l=1}^d dt \Bigg \vert x_0\right]\\
        &= \left[\left \langle \mathbb{E}\left[\int_{t_j}^{t_{j+1}}\varphi(x_t)\varphi(x_t)^T dt \Bigg \vert x_0 \right], Q_{kl} \right\rangle_{\mathbb{R}^{p \times p}}\right]_{k,l=1}^d
    \end{align}
\end{proof}
Set
\begin{equation}
    M_j := \mathbb{E}\left[\int_{t_j}^{t_{j+1}}\varphi(x_t)\varphi(x_t)^Tdt \Bigg \vert x_0 \right]
\end{equation}
which is a $(p,p)$ matrix with $M_j = M_j^T$ and $M_j \succeq 0$ for $j=0,\ldots,n-1$.  We refer to the matrices $M_j$ as {\em occupation kernels}.  They are operators, in contrast to the functions $L_i^* v$ of section \ref{sec:est_drift} and equivalent functions $\ell_{i,v}^*$ of appendix \ref{sec:drift_derivation}.

\begin{proposition}
    The cost function in \eqref{eq:diff_outer_cost} may be written as a minimization problem over matrices $Q \in \mathbb{R}^{pd \times pd}$ as follows:
    \begin{equation}
        J_{drift}(Q) = \frac{1}{n}\sum_{i=0}^{n-1}\left\|\left[\langle M_i, Q_{kl} \rangle_{\mathbb{R}^{p \times p}}\right]_{k.l=1}^d - z_i\right\|^2_{\mathbb{R}^{d \times d}} + \lambda_\sigma \|Q\|^2_{\mathbb{R}^{pd \times pd}}
    \end{equation}
    subject to the constraints that $Q^T = Q$ and $Q \succeq 0$.
\end{proposition}
\begin{proof}
Using proposition \ref{prop:matrix_of_inner_products}, substitute the expression $[\langle M_j, Q_{kl} \rangle_{\mathbb{R}^{p \times p}}]_{k,l=1}^d$ for $\mathbb{E}\left[\int_{t_j}^{t_{j+1}} a(x_t) dt \Bigg \vert x_0\right]$ appearing in the cost.  It remains to show that
\begin{equation}
    \|a\|^2_{{H'}^{d\times d}} = \|Q\|^2_{\mathbb{R}^{pd \times pd}}
\end{equation}
By theorem \ref{thm:parametrization_for_diff_explicit}, each component of $a$ is in the RKHS with kernel given by $k(x,y)^2 := (\varphi(x)^T \varphi(y))^2$.  Thus,
\begin{equation}
    [a(x)]_{kl} = \sum_{j=1}^\infty k(x, x_j^{kl})^2 \beta_j^{kl}
\end{equation}
for some $\{\beta^{kl}_j\}_{j=1}^\infty \subset  \mathbb{R}$ and $\{x_j^{kl}\}_{j=1}^\infty \subset \mathbb{R}^d$  for $k,l=1,\ldots,d$.  This implies
\begin{equation}
    Q_{kl} = \sum_{j=1}^\infty \varphi(x_j^{kl})\varphi(x_j^{kl})^T \beta_j^{kl}
\end{equation}

We get
\begin{align}
    \|a\|^2_{{H'}^{d\times d}} &=\sum_{k,l=1}^d \langle [a]_{kl}, [a]_{kl} \rangle_{H'}\\
    &=\sum_{k,l=1}^d
    \left \langle \sum_{j=1}^\infty k(\cdot, x_j^{kl})^2  \beta_j^{kl}, \sum_{r=1}^\infty k(\cdot, x_r^{kl})^2  \beta_r^{kl} \right\rangle_{H'}\\
    &=
    \sum_{k,l=1}^d\sum_{j,r = 1}^\infty \beta_j^{kl} (\varphi(x_j^{kl})^T \varphi(x_r^{kl}))^2 \beta_r^{kl}\\
   &= \sum_{k,l=1}^d \left \langle \sum_{j=1}^\infty \varphi(x_j^{kl}) \varphi(x_j^{kl})^T \beta_j^{kl}, \sum_{r=1}^\infty \varphi(x_r^{kl})\varphi(x_r^{kl})^T \beta_r^{kl}\right\rangle_{\mathbb{R}^{p \times p}}\\
   &= \sum_{k,l=1}^d \langle Q_{kl}, Q_{kl}\rangle_{\mathbb{R}^{p\times p}}\\
   &= \langle Q, Q \rangle_{\mathbb{R}^{pd \times pd}}\\
   &= \|Q\|^2_{\mathbb{R}^{pd \times pd}}
\end{align}

\end{proof}

Applying the above, we see that \eqref{eq:diff_outer_cost} may be written as a minimization problem over matrices $Q \in \mathbb{R}^{pd \times pd}$:
\begin{equation}
    J_{diff}(Q) = \frac{1}{n}\sum_{j=0}^{n-1}\left\|\left[\langle M_j, Q_{kl} \rangle_{\mathbb{R}^{p \times p}}\right]_{k,l=1}^d - z_j\right\|^2_{\mathbb{R}^{d\times d}} + \lambda_\sigma \|Q\|^2_{\mathbb{R}^{pd \times pd}}
    \label{eq:diff_cost_Q}
\end{equation}
 subject to $Q \succeq 0$. 
The cost \eqref{eq:diff_cost_Q} may be minimized as a semi-definite program (SDP).  However, as an alternative, we may apply Fenchel duality \cite{learning_psd_valued_Muzellec}.  This gives us a dual cost function which is more efficient to optimize in the case of large dimension $d$ and large number of features $p$  We make use of the following theorem:
\begin{theorem}[Fenchel duality theorem]
    Let $L(Q) = \theta(RQ) + \Omega(Q)$ be a function where $\theta \colon \mathcal{X} \to (-\infty, \infty]$ and $\Omega \colon \mathcal{Y} \to (-\infty, \infty]$ are lower semi-continuous and convex, $R \colon \mathcal{Y} \to \mathcal{X}$ is bounded linear, and $\mathcal{X}, \mathcal{Y}$ are Banach spaces.  Suppose that $0 \in \text{\rm core}(\text{\rm dom } \theta - R\;\text{\rm dom }\Omega)$.  Then  the following holds:
    \begin{equation}
       \inf_{Q \in \mathcal{Y}}\left\{ \theta(RQ) + \Omega(Q) \right\} = \sup_{\beta \in \mathcal{X}^*}\{ -\theta^*(\beta) - \Omega^*(-R^* \beta)\} 
    \end{equation}
    where $\theta^* \colon \mathcal{X}^* \to [-\infty, \infty], \Omega^* \colon \mathcal{Y}^* \to [-\infty, \infty]$ are the Fenchel conjugates of $\theta, \Omega$, respectively, and $R^* \colon \mathcal{X}^* \to \mathcal{Y}^*$ is the adjoint of $R$.
    Furthermore, if $\beta^*$ is the maximizer of the right-hand side and $\Omega^*$ is differentiable, then $\nabla \Omega^*(-R^*\beta^*)$ is the minimizer of the left-hand side.
    \label{thm:fenchel_duality}
\end{theorem}
\begin{proof}
    See \cite{Borwein2005}[Theorem 3.3.5]. 
\end{proof}

\begin{theorem}
\label{thm:fenchel_dual_cost_explicit_kernel}
   The minimizer of the Fenchel dual of the problem in \eqref{eq:diff_cost_Q} is given by
   \begin{multline}
   \label{eq:fenchel_dual_explicit}
       \beta^* :=\\ \underset{\beta \in \mathbb{R}^{n \times d \times d}}{\rm argmin}\left\{\frac{n}{4}\|\beta\|^2_{\mathbb{R}^{n \times d \times d}} + \langle \beta, z \rangle_{\mathbb{R}^{n \times d \times d}} + \frac{1}{4\lambda_\sigma}\left\|\left[\left[\sum_{j=0}^{n-1}M_j \beta_{j, kl}\right]_{k,l=1}^d\right]_-\right\|^2_{\mathbb{R}^{pd \times pd}}\right\}
   \end{multline}
   where $[A]_- = [U\Sigma U^T]_- := U\max\{-\Sigma, 0\}U^T$.  The optimal $Q^*$ minimizing \eqref{eq:diff_cost_Q} is then given by
   \begin{equation}
       Q^* = \frac{1}{2\lambda_\sigma}\left[\left[\sum_{j=0}^{n-1}M_j \beta_{j, kl}^*\right]_{k,l=1}^d\right]_-
   \end{equation}
\end{theorem}
\begin{proof}
In the notation of theorem \ref{thm:fenchel_duality}, let $R \colon \mathbb{R}^{pd \times pd} \to \mathbb{R}^{n \times d \times d}$ be defined as
\begin{equation}
    RQ := \left[\left[\langle M_j, Q_{kl} \rangle_{\mathbb{R}^{p \times p}}\right]_{j=0}^{n-1}\right]_{k,l=1}^d \in \mathbb{R}^{n \times d \times d}
\end{equation}
The derivation of $R^* \colon \mathbb{R}^{n \times d \times d} \to \mathbb{R}^{pd \times pd}$ is as follows:
\begin{multline}
    \langle \beta, RQ \rangle_{\mathbb{R}^{n \times d \times d}} = \sum_{j=0}^{n-1}\sum_{k,l=1}^d \beta_{j,kl}\text{Tr}(M_j Q_{kl}) =\\\sum_{k,l=1}^d \text{Tr}\left(\sum_{j=0}^{n-1}\beta_{j,kl}M_j Q_{kl}\right) =\\ \left\langle \left[\sum_{j=0}^{n-1} M_j \beta_{j,kl}\right]_{k,l=1}^d, Q \right \rangle_{\mathbb{R}^{pd \times pd}}
\end{multline}
Thus, we have that $R^* \beta$ is the $(pd, pd)$ block matrix whose $(k,l)$ block is given by
\begin{equation}
    [R^*\beta]_{kl} = \sum_{j=0}^{n-1} M_j \beta_{j,kl} \in \mathbb{R}^{p \times p}
\end{equation}

Let $\Omega \colon \mathbb{R}^{pd \times pd} \to (-\infty, \infty]$ be defined as
\begin{equation}
    \Omega(Q) := \begin{cases}
        \lambda_\sigma \|Q\|_{\mathbb{R}^{pd \times pd}}^2 & \text{if } Q \succeq 0\\
        \infty & \text{else}
    \end{cases}
\end{equation}
and $\theta \colon \mathbb{R}^{n \times d \times d} \to (-\infty, \infty]$ as
\begin{equation}
    \theta(\beta) := \frac{1}{n}\|\beta - z\|^2_{\mathbb{R}^{n \times d \times d}}
\end{equation}
where $z := [z_0, \ldots, z_{n-1}] \in \mathbb{R}^{n \times d \times d}$.  Then we may write \eqref{eq:diff_cost_Q} as
\begin{equation}
    L_2(Q) = \theta(RQ) + \Omega(Q)
\end{equation}
with no constraints.
By theorem \ref{thm:fenchel_duality}, we may derive \eqref{eq:fenchel_dual_explicit} using the definition of Fenchel conjugate applied to the functions $\theta$ and $\Omega$.   
   If $\beta^*$ is optimal for \eqref{eq:fenchel_dual_explicit}, then, by theorem \ref{thm:fenchel_duality}, the optimal $Q^*$ solving \eqref{eq:diff_cost_Q} is given by
\begin{equation}
    Q^* = \nabla \Omega^*(-R^*\beta^*) = \frac{1}{2\lambda_\sigma}\left[\left[\sum_{j=0}^{n-1}M_j \beta_{jkl}\right]_{k,l=1}^d\right]_-
\end{equation}
\end{proof}

We refer the interested reader to \cite{Malick2006} for a reference on how to compute the gradient $\nabla \Omega^*$ by computing the gradient of the projection operator $P_{S^{pd}_+}$ for the cone of symmetric, positive semi-definite matrices.
\section{Proofs of statements in section \ref{sec:est_diffusion}}
\label{sec:implicit_diffusion_appendix}
Recall that we use the parametrization from \cite{nonnegative_funcs_marteau2020}, \cite{learning_psd_valued_Muzellec} of $a$ given by
\begin{equation}
    a(x) := [\langle k(\cdot, x), C_{kl} k(\cdot, x) \rangle_H]_{k,l=1}^d
\end{equation}
where $C_{kl} \in \mbox{\rm Hom}_{HS}(H)$, the space of Hilbert-Schmidt operators on $H$.  Let us denote $\varphi(x) := k(\cdot, x)$.  We define the operator $C \colon H^d \to H^d$ as the matrix of operators $C_{kl}$, whose evaluation is given by
\begin{equation}
    C\mathbf{g} = \begin{bmatrix}
        \sum_{l=1}^d C_{1l}g_l \\ \vdots \\ \sum_{l=1}^d C_{dl}g_l
    \end{bmatrix} \mbox{ for } \mathbf{g} \in H^d
\end{equation}
We require $C^* = C$.  To enforce the constraint that $a(x) \succeq 0$ for all $x \in \mathbb{R}^d$, we require $C \succeq 0$ in the sense that
\begin{equation}
    \mathbf{g}^* C \mathbf{g} = \sum_{k,l=1}^d \langle g_k, C_{kl} g_l \rangle_H \geq 0 \mbox{ for all } \mathbf{g}\in H^d
\end{equation}

For $f,g \in H$, we denote by $f \otimes g$ the rank-one operator in $\mbox{Hom}_{HS}(H)$ given by $(f \otimes g)h := \langle g, h \rangle_H f$.

\begin{theorem}
\label{thm:parametrization_for_diff_implicit}
    The space of functions with parametrization according to $x \mapsto  \langle \varphi(x), D\varphi(x) \rangle_H$ with $D \in \mbox{\rm Hom}_{HS}(H)$ and $D^* = D$ and the RKHS $H'$ with kernel $k'(x,y) := k(x,y)^2$ coincide, provided the following regularity condition holds:
    \begin{equation}
        \{D \in \mbox{\rm Hom}_{HS}(H) \colon D^* = D\} = \overline{\mbox{\rm span}\{\varphi(x) \otimes \varphi(x) \colon x \in \mathbb{R}^d\}}
    \end{equation}
    where the closure is taken in the space $\mbox{\rm Hom}_{HS}(H)$. Furthermore, if $g,h \in H'$ correspond to $D, E \in \mbox{\rm Hom}_{HS}(H)$, then $\langle g, h \rangle_{H'} = \langle D, E \rangle_{HS}$.  In particular, each component of $a$ is in $H'$.
\end{theorem}
\begin{proof}
    Suppose that $g(x) = \sum_{j=1}^\infty k(x, x_j)^2 \beta_j \in H'$ for some $\{x_j\}_{j=1}^\infty \subset \mathbb{R}^d$, $\{\beta_j\}_{j=1}^\infty \subset \mathbb{R}$.  Then we get
    \begin{align}
        \sum_{j=1}^\infty k(x, x_j)^2 \beta_j &= \sum_{j=1}^\infty \langle k(\cdot, x), (k(\cdot, x_j) \otimes k(\cdot, x_j)) \beta_j k(\cdot, x) \rangle_H\\
        &= \left \langle k(\cdot, x), \left[\sum_{j=1}^\infty k(\cdot, x_j) \otimes k(\cdot, x_j) \beta_j \right]k(\cdot, x)\right\rangle_H\\
        &= \langle \varphi(x), D\varphi(x) \rangle_H
    \end{align}
    where $D := \sum_{j=1}^\infty k(\cdot, x_j) \otimes k(\cdot, x_j)\beta_j \in \mbox{\rm Hom}_{HS}(H)$.  Conversely, suppose that $g(x) := \langle \varphi(x), D \varphi(x) \rangle_H$ for some $D \in \mbox{\rm Hom}_{HS}(H)$.  By the regularity condition, there exist $\{x_j\}_{j=1}^\infty \subset \mathbb{R}^d, \{\beta_j\}_{j=1}^\infty \subset \mathbb{R}$ such that
    \begin{equation}
        D = \sum_{j=1}^\infty \varphi(x_j) \otimes \varphi(x_j) \beta_j
    \end{equation}
    Thus,
    \begin{align}
        \langle \varphi(x), D \varphi(x) \rangle_H &= \left \langle \varphi(x), \left[\sum_{j=1}^\infty \varphi(x_j) \otimes \varphi(x_j) \beta_j \right] \varphi(x) \right \rangle_H\\
        &= \sum_{j=1}^\infty \langle \varphi(x), \varphi(x_j) \rangle_H^2 \beta_j\\
        &= \sum_{j=1}^\infty k(x, x_j)^2 \beta_j
    \end{align}
    showing that $g \in H'$.  Lastly, suppose that $g(x) = \langle \varphi(x), D \varphi(x) \rangle_H$ in $H'$ with $D = \sum_{j=1}^\infty \varphi(x_j) \otimes \varphi(x_j) \beta_j$ and $h(x) = \langle \varphi(x), E \varphi(x) \rangle_H$ in $H'$ with $E = \sum_{r=1}^\infty \varphi(y_r) \otimes \varphi(y_r) \delta_r$. Then
    \begin{align}
       \langle g, h \rangle_{H'} &= \sum_{j,r=1}^\infty k(x_j, y_r)^2 \beta_j \delta_r \\
       &= \sum_{j,r=1}^\infty \langle k(\cdot, x_j), k(\cdot, y_r)\rangle_H^2 \beta_j \delta_r\\
       &= \sum_{j,r=1}^\infty \langle k(\cdot, x_j) \otimes k(\cdot, x_j), k(\cdot, y_r) \otimes k(\cdot, y_r) \rangle_{HS} \beta_j \delta_r\\
       &= \left \langle \sum_{j=1}^\infty k(\cdot, x_j) \otimes k(\cdot, x_j) \beta_j, \sum_{r=1}^\infty k(\cdot, y_r) \otimes k(\cdot, y_r)\delta_r \right \rangle_{HS}\\
       &= \langle D, E \rangle_{HS}
    \end{align}
\end{proof}
{\em Proof of proposition \ref{prop:operator_occ_kerns}}:

Suppose the regularity condition of theorem \ref{thm:parametrization_for_diff_implicit} holds.  Suppose additionally that we have the following:
\begin{equation}
    \mathbb{E}\left[\int_{t_i}^{t_{i+1}}k(x_t, x_t)^2 dt \Bigg \vert x_0 \right] < \infty
\end{equation}
for $i=0,\ldots, n-1$.  Let us define functionals $S_i \colon H' \to \mathbb{R}$ by
\begin{equation}
    S_i(g) := \mathbb{E}\left[\int_{t_i}^{t_{i+1}}g(x_t) dt \Bigg \vert x_0 \right]
\end{equation}
for $i=0,\ldots, n-1$.  Let us show that $S_i$ is bounded:
\begin{align}
    \lvert S_i(g) \rvert^2 &\leq \mathbb{E}\left[\left\lvert \int_{t_i}^{t_{i+1}}g(x_t) dt \right\rvert^2 \Bigg \vert x_0 \right]\\
    &\leq \mathbb{E}\left[\int_{t_i}^{t_{i+1}}\lvert g(x_t) \rvert^2 dt \Bigg \vert x_0 \right](t_{i+1} - t_i)\\
    &=\mathbb{E}\left[\int_{t_i}^{t_{i+1}}\langle g, k(\cdot, x_t)^2 \rangle_{H'}^2 dt \Bigg \vert x_0 \right](t_{i+1} - t_i)\\
    &\leq \|g\|^2_{H'} \mathbb{E}\left[\int_{t_i}^{t_{i+1}}k(x_t, x_t)^2 dt \Bigg \vert x_0 \right](t_{i+1} - t_i)
\end{align}
where we have used Jensen's inequality on the first line, Cauchy-Schwarz on the second, the reproducing property of $H'$ on the third, and Cauchy-Schwarz again on the fourth.  Thus, the regularity condition shows that $S_i$ is bounded for $i=0,\ldots, n-1$.  By the Riesz representation theorem, there exists $\gamma_i^* \in H'$ with
\begin{equation}
    S_i(g) = \langle g, \gamma_i^* \rangle_{H'}
\end{equation}
We get
\begin{align}
    \gamma_i^*(x) &= \langle k(\cdot, x)^2, \gamma_i^* \rangle_{H'}\\
    &= S_i(k(\cdot, x)^2)\\
    &= \mathbb{E}\left[\int_{t_i}^{t_{i+1}}k(x,x_t)^2 dt \Bigg \vert x_0 \right]\\
    &= \mathbb{E}\left[\int_{t_i}^{t_{i+1}} \langle k(\cdot, x), (k(\cdot, x_t) \otimes k(\cdot, x_t)) k(\cdot, x) \rangle_H dt \Bigg \vert x_0 \right]\\
    &= \left \langle k(\cdot, x), \mathbb{E}\left[\int_{t_i}^{t_{i+1}}k(\cdot, x_t) \otimes k(\cdot, x_t) dt \Bigg \vert x_0 \right] k(\cdot, x) \right\rangle_{HS}
\end{align}
Since $\gamma_i^* \in H'$, this shows that $\mathbb{E}\left[\int_{t_i}^{t_{i+1}} k(\cdot, x_t)\otimes k(\cdot, x_t) dt \Bigg \vert x_0 \right] \in \mbox{\rm Hom}_{HS}(H)$ by theorem \ref{thm:parametrization_for_diff_implicit}.  Another application of theorem \ref{thm:parametrization_for_diff_implicit} shows that
\begin{align}
    \mathbb{E}\left[\int_{t_i}^{t_{i+1}}a(x_t) dt \Bigg \vert x_0 \right] &= [\langle \gamma_i^*, [a]_{kl} \rangle_{H'}]_{k,l=1}^d\\
    &=\left[\left \langle \mathbb{E}\left[\int_{t_i}^{t_{i+1}}\varphi(x_t) \otimes \varphi(x_t) dt \Bigg \vert x_0 \right], C_{kl} \right\rangle_{HS}\right]_{k,l=1}^d
\end{align}
\hfill $\Box$

Let
\begin{equation}
    M_i := \mathbb{E}\left[\int_{t_i}^{t_{i+1}}\varphi(x_t) \otimes \varphi(x_t) \Bigg \vert x_0\right] \in \mbox{\rm Hom}_{HS}(H)
\end{equation}
for $i=0,\ldots, n-1$.  We refer to the operators $M_i$ as {\em occupation kernels}.  They satisfy $M_i = M_i^*$ and $M_i \succeq 0$.

In order to adapt the representer theorem of \cite{learning_psd_valued_Muzellec}, we must approximate the $M_i$ with finite-rank operators, denoted $\hat{M}_i$.  We define them as
\begin{equation}
    \hat{M}_i := \frac{t_{i+1} - t_i}{2M}\sum_{u=1}^M \left[\varphi(y_i^{(u)}) \otimes \varphi(y_i^{(u)}) + \varphi(y_{i+1}^{(u)}) \otimes \varphi(y_{i+1}^{(u)})\right]
\end{equation}
where we have used a trapezoid-rule integral quadrature and Monte Carlo estimate of the expectation.

{\em Proof of theorem \ref{thm:diff_representer_thm}}:

Let $J'_{diff}$ be the loss function in equation \eqref{eq:diff_cost_with_M_i_hat}.  Let $\pi_n \colon H \to H$ denote orthogonal projection onto the finite-dimensional subspace $V$ of $H$ spanned by $k(\cdot, y_i^{(u)})$ for $i=0,\ldots,n$ and $u=1,\ldots, M$.  Then we get
\begin{align}
    [\langle \hat{M}_i, \pi_n C_{kl} \pi_n \rangle_{HS}]_{k,l=1}^d &= [\langle \pi_n \hat{M}_i \pi_n, C_{kl} \rangle_H]_{k,l=1}^d\\
    &= \scriptscriptstyle{\left[\left \langle \pi_n \frac{t_{i+1} - t_i}{2M}\sum_{u=1}^M \left[\varphi(y_i^{(u)}) \otimes \varphi(y_i^{(u)}) + \varphi(y_{i+1}^{(u)}) \otimes \varphi(y_{i+1}^{(u)}) \right] \pi_n, C_{kl} \right \rangle_{HS}\right]_{k,l=1}^d}\\
    &= \scriptstyle{\left[\left \langle \frac{t_{i+1} - t_i}{2M}\sum_{u=1}^M \left[\varphi(y_i^{(u)}) \otimes \varphi(y_i^{(u)}) + \varphi(y_{i+1}^{(u)}) \otimes \varphi(y_{i+1}^{(u)}) \right], C_{kl} \right \rangle_{HS}\right]_{k,l=1}^d}\\
    &= [\langle \hat{M}_i, C_{kl} \rangle_{HS}]_{k,l=1}^d
\end{align}
using the self-adjoint property of projection operators.  Let $\Pi_n \colon H^d \to H^d$ denote component-wise orthogonal projection onto $V$ and $\Pi^\perp_n \colon H^d \to H^d$ the analogous operator onto $V^\perp$.  Then
\begin{align}
    \|C\|^2_{HS} &= \|(\Pi_n + \Pi^\perp_n)C(\Pi_n + \Pi^\perp_n)\|^2_{HS}\\
    &= \|\Pi_n C \Pi_n\|^2_{HS} + \|\Pi_n C \Pi^\perp_n\|^2_{HS} + \|\Pi_n^\perp C \Pi_n\|^2_{HS} + \|\Pi_n^\perp C \Pi_n^\perp\|^2_{HS}\\
    &\geq \|\Pi_n C \Pi_n\|^2_{HS}
\end{align}
These two facts together imply that $J'_{diff}(\Pi_n C \Pi_n) \leq J'_{diff}(C)$ for any operator $C \in \mbox{\rm Hom}_{HS}(H)$ with $C^* = C$.
\hfill $\Box$

Let $\psi \in \mathcal{L}(H, \mathbb{R}^{(n+1)M})$ be defined as
\begin{equation}
    \psi := \begin{bmatrix}
        k(\cdot, y_0^{(1)})\\
        k(\cdot, y_0^{(2)})\\
        \vdots
        \\
        k(\cdot, y_n^{(M)})
    \end{bmatrix}
\end{equation}
so that
\begin{equation}
    \psi g = \begin{bmatrix}
        \langle k(\cdot, y_0^{(1)}), g \rangle_H \\ \langle k(\cdot, y_0^{(2)}), g \rangle_H \\ \vdots \\ \langle k(\cdot, y_n^{(M)}), g \rangle_H
    \end{bmatrix}=\begin{bmatrix}
        g(y_0^{(1)}) \\ g(y_0^{(2)}) \\ \vdots \\ g(y_n^{(M)})
    \end{bmatrix}
\end{equation}
Then we have the adjoint, $\psi^* \in \mathcal{L}(\mathbb{R}^{(n+1)M}, H)$ given by
\begin{equation}
    \psi^* \beta = \sum_{j=0}^n \sum_{u=1}^M k(\cdot, y_j^{(u)}) \beta_{(j, u)}
\end{equation}
where we use double-index notation for the entries of $\beta$. This means that
\begin{equation}
    (j,u) := (j-1)M + u
\end{equation}
so that $\beta \in \mathbb{R}^{(n+1)M}$.
Define
\begin{equation}
    \psi(x) := \psi \varphi(x) = \begin{bmatrix}
        k(x, y_0^{(1)}) \\ \vdots \\ k(x, y_n^{(M)})
    \end{bmatrix}
\end{equation}
so that we may think of $\psi$ as a function $\mathbb{R}^d \to \mathbb{R}^{(n+1)M}$.  Let $K \in \mathbb{R}^{(n+1)M \times (n+1)M}$ be such that
\begin{equation}
    K_{(i, u), (j, v)} := k(y_i^{(u)}, y_j^{(v)})
\end{equation}

\begin{proposition}
    Considering $\psi \psi^*$ as an element of $\mathbb{R}^{(n+1)M \times (n+1)M}$, we get
    \begin{equation}
        \psi \psi^* = K
    \end{equation}
    where
    \begin{equation}
        K_{(i, u), (j,v)} := k(y_i^{(u)}, y_j^{(v)})
    \end{equation}
    \label{prop:psi_psi_star_is_K}
\end{proposition}
\begin{proof}
    Let $v \in \mathbb{R}^{(n+1)M}$.  Then
    \begin{align}
        \psi \psi^* v &= \psi \left(\sum_{j=0}^n \sum_{u=1}^M k(\cdot, y_j^{(u)})v_{(j, u)}\right)\\
        &= \begin{bmatrix}
            \sum_{j=0}^n \sum_{u=1}^M k(y_0^{(1)}, y_j^{(u)})v_{(j,u)}
            \\
            \vdots
            \\
            \sum_{j=0}^n \sum_{u=1}^M k(y_n^{(M)}, y_j^{(u)})v_{(j,u)}
        \end{bmatrix}
        \\
        &= Kv
    \end{align}
\end{proof}
\begin{proposition}
\label{prop:projection_formula}
    We have that
    \begin{equation}
        \psi^* K^{+} \psi = \pi_n
    \end{equation}
    where $\pi_n$ is orthogonal projection onto the span of $\{k(\cdot, y_j^{(u)})\}_{(j,u)=(0,1)}^{(n,M)}$ and $K^+$ is the Moore-Penrose inverse of $K$.
\end{proposition}
\begin{proof}
    Let $g \in H$.  By applying the definition of $\psi^*$ we have that
    \begin{equation}
        \pi_n g = \psi^* w
    \end{equation}
    for some $w \in \mathbb{R}^{(n+1)M}$.  We may find $w$ by solving the following system of equations:
    \begin{equation}
        \left \langle g - \sum_{j=0}^n \sum_{u=1}^M k(\cdot, y_j^{(u)}) w_{(j, u)}, k(\cdot, y_i^{(v)})\right\rangle_H = 0
    \end{equation}
    for $i=0,\ldots, n, v=1,\ldots, M$.  This is equivalent to solving
    \begin{equation}
        Kw = \psi g
    \end{equation}
    This system is consistent so that $w = K^{+}\psi g$ is a solution.  The result follows.
\end{proof}
Let $R \in \mathbb{R}^{r \times (n+1)M}$ be a matrix such that $K = R^T R$, where $r := \mbox{rank}(K)$.  Note that such a matrix necessarily exists because $K \succeq 0, K^T = K$.  Then define $\gamma \in \mathcal{L}(H, \mathbb{R}^r)$ by
\begin{equation}
    \gamma := (RR^T)^{-1}R\psi
\end{equation}
Note that in the case where $K$ is full-rank, this is equal to
\begin{equation}
    \gamma = R^{-T}\psi
\end{equation}
\begin{proposition}
\label{prop:gamma_identities}
    We have
    \begin{equation}
        \gamma \gamma^* = I_r, \;\;\; \gamma^* \gamma = \pi_n
    \end{equation}
    where $I_r$ is the (r,r) identity matrix and $\pi_n$ is orthogonal projection onto the span of $\{k(\cdot, y_j^{(u)})\}_{(j,u) = (0, 1)}^{(n,M)}$. 
    \label{prop:gamma_prop}
\end{proposition}
\begin{proof}
By proposition \ref{prop:psi_psi_star_is_K}, we have
    \begin{align}
        \gamma \gamma^* &= (RR^T)^{-1}R\psi \psi^* R^T (RR^T)^{-1}\\
        &= (RR^T)^{-1}RKR^T (RR^T)^{-1}\\
        &= (RR^T)^{-1}R(R^T R) R^T(RR^T)^{-1}\\
        &= I_r
    \end{align}
By proposition \ref{prop:projection_formula}, the projection $\pi_n$ is given by $\pi_n = \psi^* K^{+} \psi$.  Thus,
\begin{align}
\gamma^* \gamma  &= \psi^* R^T(RR^T)^{-2}R \psi  \\
&= \psi^* R^+ R^{+T}\psi \\
&= \psi^* (R^T R)^+\psi \\
&= \psi^* K^{+} \psi \\
&= \pi_n 
   \end{align}
\end{proof}
Define
\begin{equation}
    \hat{\gamma} := I_d \otimes \gamma \in \mathcal{L}(H^d, \mathbb{R}^{rd}), \;\;\; \hat{R} := I_d\otimes R \in \mathbb{R}^{rd \times rd}
\end{equation}
and $\gamma(x) := \gamma \varphi(x)$.

\begin{lemma}
\label{lemma:gamma_identities_multivariate}
    We have
    \begin{equation}
    \hat{\gamma}\hat{\gamma}^* = I_{rd},\;\;\;
        \hat{\gamma}^* \hat{\gamma} = \Pi_n
    \end{equation}
\end{lemma}
\begin{proof}
    By proposition \ref{prop:gamma_identities}, we get
    \begin{align}
\hat{\gamma}\hat{\gamma}^* &= I_d \otimes \gamma \gamma^*\\
        &= I_d \otimes I_r\\
        &= I_{rd}
    \end{align}
    and
    \begin{align}
    \hat{\gamma}^*\hat{\gamma} &= I_d \otimes \gamma^* \gamma\\
    &= I_d \otimes \pi_n\\
    &= \Pi_n
    \end{align}
\end{proof}

{\em Proof of theorem \ref{thm:isometry}}:

By lemma \ref{lemma:gamma_identities_multivariate}, we have
\begin{align}
    \Pi_n C \Pi_n &= \hat{\gamma}^*\hat{\gamma} C \hat{\gamma}^*\hat{\gamma}\\
    &= \hat{\gamma}^* A \hat{\gamma}
\end{align}
where $A := \hat{\gamma}C\hat{\gamma}^* \in \mathbb{R}^{rd \times rd}$.  Note that $A^T = A$.  Also,
\begin{align}
    \hat{\gamma}^* A \hat{\gamma} &= \Pi_n \hat{\gamma}^* A \hat{\gamma}\Pi_n\\
    &= \Pi_n C \Pi_n
\end{align}
where $C := \hat{\gamma}^* A \hat{\gamma} \in \mbox{\rm Hom}_{HS}(H)$.  We have $C^* = C$.

We compute
\begin{align}
    \langle \hat{M}_i, C_{kl} \rangle_{HS} &=
    \langle \hat{M}_i, \gamma^* A_{kl} \gamma \rangle_{HS} \\
    &= \langle \gamma \hat{M}_i \gamma^*, A_{kl} \rangle_{\mathbb{R}^{r \times r}}\\
    &= \left \langle \gamma \frac{t_{i+1} - t_i}{2M}\sum_{u=1}^M\left[\varphi(y_i^{(u)}) \otimes \varphi(y_i^{(u)}) + \varphi(y_{i+1}^{(u)})\otimes \varphi(y_{i+1}^{(u)})\right]\gamma^*, A_{kl} \right\rangle_{\mathbb{R}^{r \times r}}\\
    &= \left \langle \frac{t_{i+1} - t_i}{2M}\sum_{u=1}^M\left[\gamma(y_i^{(u)})\gamma(y_i^{(u)})^T + \gamma(y_{i+1}^{(u)})\gamma(y_{i+1}^{(u)})^T\right], A_{kl} \right\rangle_{\mathbb{R}^{r \times r}}\\
    &= \langle N_i, A_{kl} \rangle_{\mathbb{R}^{r \times r}}
\end{align}
And
\begin{align}
    \|\hat{\gamma}^* A \hat{\gamma}\|^2_{HS} &= \mbox{\rm Tr}\left(\hat{\gamma}^* A \hat{\gamma}\hat{\gamma}^*A \hat{\gamma}\right)\\
    &= \mbox{\rm Tr}\left(\hat{\gamma}\hat{\gamma}^* A \hat{\gamma}\hat{\gamma}^* A\right)\\
    &=\mbox{\rm Tr}(I_{rd}AI_{rd}A)\\
    &=\|A\|^2_{\mathbb{R}^{rd \times rd}}
\end{align}
\hfill $\Box$

In this way, we have reduced the problem to the case of an explicit kernel with $k''(x,y) := \gamma(x)^T\gamma(y)$.  Thus, the theorems for the explicit kernel case apply here.

{\em Proof of theorem \ref{thm:fenchel_dual_implicit_kernel}}:

See the proof of theorem \ref{thm:fenchel_dual_cost_explicit_kernel} in appendix \ref{sec:derivation_outer_diff} with $A$ in place of $Q$ and $N_i$ in place of $M_i$.

\hfill $\Box$
\end{document}